\documentclass[wcp,twoside,11pt]{jmlrclean}

\jmlrvolume{Work in progress by Shie Mannor, Vianney Perchet and Gilles Stoltz}
\jmlryear{(2016)}
\title[Approachability in Unknown Games]{Approachability in Unknown Games: \\ Online Learning Meets Multi-Objective Optimization}

\renewcommand{\leq}{\leqslant}
\renewcommand{\geq}{\geqslant}

\newcommand{\R}{\mathbb{R}}
\newcommand{\cA}{\mathcal{A}}
\newcommand{\cB}{\mathcal{B}}
\newcommand{\cC}{\mathcal{C}}
\newcommand{\cP}{\mathcal{P}}
\newcommand{\cT}{\mathcal{T}}
\newcommand{\cG}{\mathcal{G}}
\newcommand{\cH}{\mathcal{H}}
\newcommand{\cK}{\mathcal{K}}
\newcommand{\cKp}{\mathcal{K}_{\mbox{\rm \tiny proj}}}
\newcommand{\cR}{\mathcal{R}}
\newcommand{\ol}{\overline}
\newcommand{\unif}{\mbox{\rm \tiny unif}}
\newcommand{\parti}{\mbox{\rm \tiny part.}}
\newcommand{\dist}{\mathrm{d}}
\newcommand{\cav}{\mathrm{cav}}
\newcommand{\conv}{\mathrm{co}}
\newcommand{\vex}{\mathrm{vex}}
\newcommand{\veccol}[2]{\left[\begin{array}{c} #1 \\ #2 \end{array}\right]}
\newcommand{\ind}{\mathbb{I}}
\newcommand{\tm}{\widetilde{m}}
\newcommand{\ty}{\widetilde{y}}
\newcommand{\tx}{\widetilde{x}}
\renewcommand{\epsilon}{\varepsilon}
\newcommand{\argmin}{\mathop{\mathrm{arg\,min}}}
\newcommand{\argmax}{\mathop{\mathrm{arg\,max}}}

\newtheorem{assumption}{Assumption}
\newcommand{\proofof}[1]{\begin{proof}{\textbf{#1}}}

\begin{document}

\author[Mannor, Perchet, Stoltz]{Shie Mannor \Email{shie@ee.technion.ac.il} \\
       \addr Israel Institute of Technology (Technion), Faculty of Electrical Engineering, 32000 Haifa, Israel
       \AND
       Vianney Perchet \Email{vianney.perchet@normalesup.org} \\
       \addr ENSAE ParisTech, 3 avenue Pierre Larousse, 92240 Malakoff, France
       \AND
       Gilles Stoltz \Email{stoltz@hec.fr} \\
       \addr GREGHEC: HEC Paris --- CNRS, 1 rue de la Lib\'eration, 78351 Jouy-en-Josas, France}

\maketitle

\begin{abstract}
In the standard setting of approachability there are two players and a target set. The players play repeatedly
a known vector-valued game where the first player wants to have the average vector-valued payoff converge to
the target set which the other player tries to exclude it from this set.
We revisit this setting in the spirit of online learning and do not assume that the first player knows the game
structure: she receives an arbitrary vector-valued reward vector at every round.
She wishes to approach the smallest (``best'') possible set given the observed average payoffs in hindsight.
This extension of the standard setting has implications even when the original target set is not
approachable and when it is not obvious which expansion of it should be approached instead.
We show that it is impossible, in general, to approach the best target set in hindsight and propose
achievable though ambitious alternative goals.
We further propose a concrete strategy to approach these goals. Our method does not
require projection onto a target set and amounts to switching between scalar regret minimization
algorithms that are performed in episodes. Applications to global cost minimization
and to approachability under sample path constraints are considered.
\end{abstract}


\section{Introduction}

The approachability theory of~\citet{Bla56} is arguably the most general approach available so far
for online multi-objective optimization and it has received significant attention recently in the learning
community (see, e.g., \citealp{BNR}, and the references therein).
In the standard setting of approachability there are two players, a vector-valued payoff function, and a target set.
The players play a repeated vector-valued game where the first player wants the average vector-valued payoff
(representing the states in which the different objectives are) to converge to the target set
(representing the admissible values for the said states), which the opponent tries to exclude.
The target set is prescribed a priori before the game starts and the aim of the decision-maker
is that the average reward be asymptotically inside the target set.

\paragraph{A theory of approachability in unknown games \textnormal{(i.e., for arbitrary vector-valued bandit problems)}.}
The analysis in approachability has been limited to date to cases where some underlying structure
of the problem is known, namely the vector payoff function $r$ (and some signalling structure
if the obtained payoffs are not observed).
We consider the case of ``unknown games'' where only vector-valued rewards are observed and
there is no a priori assumption on what can and cannot be obtained. In particular, we do not assume that
there is some underlying game structure we can exploit.
In our model at each round, for every action of the decision maker there is a vector-valued reward that is
only assumed to be arbitrary.
The minimization of regret could be extended to this setting (see, e.g., \citealp[Sections~7.5 and~7.10]{CeLu06}).
And we know that the minimization of regret is a special case of approachability. Hence our motivation
question: can a theory of approachability be developed for unknown games?

One might wonder if it is possible to treat an unknown game as a known game with a very large
class of actions and then use approachability. While such lifting is possible in principle, it would
lead to unreasonable time and memory complexity as the dimensionality of the problem will explode.

In such unknown games, the decision maker does not try to approach a pre-specified target set, but
rather tries to approach the best (smallest) target set given the observed vector-valued rewards.
Defining a goal in terms of the actual rewards is standard in online learning, but has not been
pursued (with a few exceptions listed below) in the multi-objective optimization community.

\paragraph{A theory of smallest approachable set in insight.}
Even in known games it may happen that no pre-specified target set is given, e.g., when
the natural target set is not approachable. Typical relaxations are then to consider uniform
expansions of this natural target set or its convex hull. Can we do better?
To answer this question, another property of regret minimization is our source of inspiration. The definition of a no-regret strategy (see, e.g., \citealp{CeLu06}) is that its performance is asymptotically as good as the best constant strategy, i.e., the strategy that selects at each stage the same mixed action. Another way to formulate this claim is that a no-regret strategy performs (almost) as well as the best mixed action in hindsight.
In the approachability scenario, this question can be translated into the existence of a strategy that
approaches the smallest approachable set for a mixed action in hindsight. If the answer is negative (and, unfortunately, it is) the next question is to define a weaker aim that would still be more ambitious
than the typical relaxations considered.

\paragraph{Short literature review.}
Our approach generalizes several existing works. Our proposed strategy can be used for standard
approachability in all the cases where the desired target set is not approachable and where one wonders
what the aim should be. We illustrate this on the problems of global costs introduced by~\citet{GlobalCosts}
and of approachability with sample path constraints as described in the special case of
regret minimization by \citet{Pathwise}.

The algorithm we present does not require projection which is the
Achilles' heel of many approachability-based schemes (it does so
similarly to \citealp{BernShim}).
Our approach is also strictly more general and more ambitious than one recently considered
by~\citet{AFFT}. An extensive comparison to the results by~\citet{BernShim}
and~\citet{AFFT} is offered in Section~\ref{sec:disc}.

\paragraph{Outline.}
This article consists of four parts of about equal lengths.
We first define the problem of approachability in unknown games
and link it to the standard setting of approachability (Section~\ref{sec:setting}).

We then discuss what are the reasonable target sets to consider
(Sections~\ref{sec:descrex} and~\ref{sec:choosetarget}). Section~\ref{sec:descrex} shows by means
of two examples that the best-in-hindsight expansion cannot be achieved while its
convexification can be attained but is not ambitious enough. Section~\ref{sec:choosetarget}
introduces a general class of achievable and ambitious enough targets:
a sort of convexification of some individual-response-based target set.

The third part of the paper (Section~\ref{sec:regretmin})
exhibits concrete and computationally efficient algorithms to achieve the goals
discussed in the first part of the paper.
The general strategy of Section \ref{sec:regretmin}
amounts to playing a (standard) regret minimization in blocks and modifying the
direction as needed; its performance and merits are then studied in detail with respect to
the literature mentioned above. It bears some resemblance with the approach developed by~\citet{BNR}.

Last but not least, the fourth part of the paper revisits two important problems,
for which dedicated methods were created and dedicated articles were written:
regret minimization with global cost functions, and online learning with sample path constraints
(Section \ref{sec:appli}). We show that our general strategy has stronger performance guarantees in
these problems than the ad hoc strategies that had been constructed by the literature.

\section{Setup (``unknown games''), notation, and aim}
\label{sec:setting}

The setting is the one of (classical) approachability, that is, vector payoffs are considered.
The difference lies in the aim. In (classical) approachability theory,
the average $\ol{r}_T$ of the obtained vector payoffs should converge asymptotically
to some target set $\cC$, which can be known to be approachable based on the existence and
knowledge of the payoff function $r$.
In our setting, we do not know whether $\cC$ is approachable because there is no underlying payoff function.
We then ask for convergence to some $\alpha$--expansion of $\cC$, where $\alpha$ should be as small as possible.

\paragraph{Setting: unknown game with vectors of vector payoffs.}
The following game is repeatedly played between two players, who will be
called respectively the decision-maker (or first player) and the opponent (or second player).
Vector payoffs in $\R^d$, where $d \geq 1$, will be considered.
The first player has finitely many actions whose set we denote by $\cA = \{1,\ldots,A\}$.
We assume $A \geq 2$ throughout the paper, to avoid trivialities.
The opponent chooses at each round $t \in \{1,2,\ldots\}$ a vector $m_t = (m_{t,a})_{a \in \cA}$ of vector payoffs
$m_{t,a} \in \R^d$. We impose the restriction that these vectors $m_t$ lie in a convex and bounded set $\cK$ of $(\R^{d})^{\cA}$.
The first player picks at each round $t$ an action $a_t \in \cA$, possibly at random according to
some mixed action $x_t = (x_{t,a})_{a \in \cA}$; we denote by $\Delta(\cA)$ the set of all such mixed actions.
She then receives $m_{t,a_t}$ as a vector payoff. We can also assume that $m_{t,a_t}$ is the only feedback
she gets on $m_t$ and that she does not see the other components of $m_t$ than the one she chose.
(This is called bandit monitoring but can and will be relaxed to a full monitoring as we explain below.)

\begin{remark}
\label{rk:K}
We will not assume that the first player knows $\cK$ (or any bound on the maximal norm of its elements);
put differently, the scaling of the problem is unknown.
\end{remark}

The terminology of ``unknown game'' was introduced in the machine learning literature,
see \citet[Sections~7.5 and 7.11]{CeLu06} for a survey. A game is unknown
(to the decision-maker) when she not only does not observe the vector
payoffs she would have received has she chosen a different pure action (bandit
monitoring) but also when she does not even know the underlying structure of the game,
if any such structure exists. Section~\ref{sec:setting2} will make the latter point
clear by explaining how the classical setting of approachability introduced by~\citet{Bla56}
is a particular case of the setting described above: some payoff function $r$ exists therein and the
decision-maker knows $r$. The strategy proposed by~\citet{Bla56} crucially relies on the knowledge
of $r$. In our setting, $r$ is unknown and even worse, might not even exist.
Section~\ref{sec:regretunknown} (and Section~\ref{sec:app:regretunknown}) will recall how a particular case of approachability
known as minimization of the regret could be dealt with for unknown games.

\paragraph{Formulation of the approachability aim.}
The decision-maker is interested in controlling her average payoff
\[
\widetilde{r}_T = \frac{1}{T} \sum_{t=1}^T m_{t,a_t}\,.
\]
She wants it to approach an as small as possible neighborhood of a given target set $\cC$, which
we assume to be closed.
This concept of neighborhood could be formulated in terms of a general filtration (see Remark~\ref{rk:2}
below); for the sake of concreteness we resort rather to expansions of a base set $\cC$ in some $\ell_p$--norm,
which we denote by $\Arrowvert\,\cdot\,\Arrowvert$, for $0 < p \leq +\infty$.
Formally, we denote by $\cC_\alpha$ the closed $\alpha$--expansion in $\ell_p$--norm of $\cC$:
\[
\cC_\alpha
= \bigl\{ c \in \R^d : \ \ \exists \, c' \in \cC \ \mbox{s.t.} \ \Arrowvert c - c' \Arrowvert_p \leq \alpha \bigr\}
= \bigl\{ c \in \R^d : \ \ \dist_p(c,\cC) \leq \alpha \bigr\}\,.
\]
Here and in the sequel, $\dist_p(\,\cdot\,,S)$ denotes the distance in $\ell_p$--norm to a set $S$.

As is traditional in the literature of approachability and regret minimization,
we consider the smallest set that would have been approachable in hindsight, that is,
had the averages of the vectors of vector payoffs be known in advance:
\[
\ol{m}_T = \frac{1}{T} \sum_{t=1}^T m_t\,, \qquad \mbox{whose components equal} \qquad
\forall \, a \in \cA, \quad \ol{m}_{T,a} = \frac{1}{T} \sum_{t=1}^T m_{t,a}\,.
\]
This notion of ``smallest set'' is somewhat tricky and the first part of this article
will be devoted to discuss it. The model we will consider
is the following one. We fix a target function $\varphi : \cK \to [0,+\infty)$;
it takes $\ol{m}_T$ as argument.
(Section~\ref{sec:choosetarget} will indicate reasonable such choices of $\varphi$.)
It associates with it the $\varphi(\ol{m}_T)$--expansion of $\cC$. Our aim is then
to ensure the almost-sure convergence
\[
\dist_p \bigl( \widetilde{r}_T, \, \cC_{\varphi(\ol{m}_T)} \bigr)
\longrightarrow 0 \qquad \mbox{as} \ \ T \to \infty\,.
\]
As in the definition of classic approachability, uniformity will be required with
respect to the strategies of the opponent: the decision-maker should construct
strategies such that for all $\varepsilon > 0$, there exists a time $T_\varepsilon$
such that for all strategies of the opponent, with probability at least $1-\varepsilon$,
\[
\sup_{T \geq T_\varepsilon} \, \dist_p \bigl( \widetilde{r}_T, \, \cC_{\varphi(\ol{m}_T)} \bigr)
\, \leq \varepsilon\,.
\]

\begin{remark}
\label{rk:2}
More general filtrations $\alpha \in [0,+\infty) \mapsto \cC_\alpha$ could have been considered than
expansions in some norm. By ``filtration'' we mean that $\cC_\alpha \subseteq \cC_{\alpha'}$
for all $a \leq \alpha'$.
For instance, if $0 \in \cC$, one could have considered shrinkages and blow-ups, that is, $\cC_0 = \{0\}$ and
$\cC_\alpha = \alpha \, \cC$ for $\alpha > 0$.
Or, given some compact set $\cB$ with non-empty interior, $\cC_\alpha = \cC + \alpha\cB$ for $\alpha \geq 0$.
But for the sake of clarity and simplicity, we restrict the exposition
to the more concrete case of expansions of a base set $\cC$ in some $\ell_p$--norm.
\end{remark}

\paragraph{Summary: the two sources of unknowness.}
As will become clearer in the concrete examples presented in Section~\ref{sec:appli},
not only the structure of the game is unknown and might even not exist (first source of
unknownness) but also the target is unknown. This second source arises also in known
games, in the following cases: when some natural target (e.g., some best-in-hindsight target)
is proven to be unachievable or when some feasible target is not ambitious enough
(e.g., the least approachable uniform expansion of $\cC$ as will be discussed in
Section~\ref{sec:opportunistic}). What to aim for, then? Convex relaxations are
often considered more manageable and ambitious enough targets;
but we will show that they can be improved upon in general.

See the paragraph ``Discussion'' on page~\pageref{par:discussion}
for more details on these two sources of unknownness in the concrete example of global costs.

\subsection{Two classical relaxations: mixed actions and full monitoring}
We present two extremely classical relaxations of the general setting described above. They come
at no cost but simplify the exposition of our general theory.

\paragraph{The decision-maker can play mixed actions.}
First, because of martingale convergence results, for instance,
the Hoeffding-Azuma inequality, controlling $\widetilde{r}_T$
is equivalent to controlling the averages $\ol{r}_T$ of the conditionally expected payoffs $r_t$, where
\[
r_t =  x_t \odot m_t = \sum_{a \in \cA} x_{t,a} m_{t,a}
\qquad \mbox{and} \qquad
\ol{r}_T = \frac{1}{T} \sum_{t=1}^T r_t = \frac{1}{T} \sum_{t=1}^T x_t \odot m_t\,.
\]
Indeed, the boundedness of $\cK$ and a component-by-component application of the
said inequality ensure that there exists a constant $C$
such that for all $\varepsilon > 0$, for all $T \geq 1$,
for all strategies of the opponent, with probability at least $1-\varepsilon$,
\[
\bigl\Arrowvert \widetilde{r}_T - \ol{r}_T \bigr\Arrowvert_p
\leq C \sqrt{\frac{\ln(d/\varepsilon)}{T}}\,.
\]
Given $\varepsilon > 0$,
we use these inequalities each with $\varepsilon$ replaced by $\varepsilon/T^2$:
a union bound entails that
choosing $T_\varepsilon \geq 2$ sufficiently large so that
\[
\sup_{T \geq T_\varepsilon} C \sqrt{\frac{\ln(d T^2/\varepsilon)}{T}} \leq \varepsilon\,,
\]
we then have, for all strategies of the opponent, with probability at least $1-\varepsilon$,
\begin{equation}
\label{eq:howtogetunif}
\sup_{T \geq T_\varepsilon} \, \bigl\Arrowvert \widetilde{r}_T - \ol{r}_T
\bigr\Arrowvert_p \leq \varepsilon\,.
\end{equation}
Therefore, we may focus on $\ol{r}_T$ instead of $\widetilde{r}_T$
in the sequel and consider equivalently the aim~\eqref{eq:aim} discussed below.

\paragraph{The decision-maker can enjoy a full monitoring.}
Second, the bandit-monitoring assumption can be relaxed to a full monitoring,
at least under some regularity assumptions, e.g., uniform continuity of the
target function~$\varphi$.

Indeed, we assumed that the decision-maker only gets to observe $m_{t,a_t}$ after choosing the component
$a_t \in \cA$ of $m_t$. However, standard estimation techniques presented by \citet{NCBBandits}
and \citet[Sections V.5.a and VI.6]{MSZ94}
provide accurate and unbiased estimators $\widehat{m}_t$ of the whole vectors $m_t$,
at least in the case when the latter only depends on what happened in the past and on
the opponent's strategy but not on the decision-maker's present\footnote{\label{fn:1} However,
such a dependency can still be dealt with in some cases, see, e.g., the
case of regret minimization in Section~\ref{sec:regretunknown} (and Section~\ref{sec:app:regretunknown}): when the dependency
on the decision-maker's present action $a_t$ comes only through an additive term
equal to the obtained payoff, which is known.}
choice of an action $a_t$. The components of these estimators $\widehat{m}_t$ equal, for $a \in \cA$,
\begin{equation}
\label{eq:estmt}
\widehat{m}_{t,a} = \frac{m_{t,a}}{x_{t,a}} \ind_{\{ a = a_t \}}\,,
\end{equation}
with the constraint on mixed actions that $x_{t,a} \geq \gamma_t$ for all $t \geq 1$.
The decision-maker should then base her decisions and apply her strategy on $\widehat{m}_t$,
and eventually choose as a mixed action the convex combination of the
mixed action she would have freely chosen based on the $\widehat{m}_t$, with weight $1-\gamma_t$, and
of the uniform distribution, with weight $\gamma_t$.

Indeed, by the Hoeffding-Azuma inequality, the averages of the
vector payoffs and of the vectors of vector payoffs based respectively
on the $m_{t,a}$ and on the $\widehat{m}_{t,a}$, as well as the corresponding
average payoffs obtained by the decision-maker, differ by something of the
order of
\[
\frac{1}{T} \left( \sqrt{\sum_{t=1}^T \frac{1}{\gamma_t^2}} + \sum_{t=1}^T \gamma_t \right)
\sqrt{\ln(1/\varepsilon)}
\]
for each $T$ with probability at least $1-\varepsilon$, and uniformly
over the opponent's strategies. These differences vanish
as $T \to \infty$, e.g., at a $T^{-2/3}$ rate when
the $\gamma_t$ are of the order of $t^{-1/3}$.
A treatment similar to the one performed to obtain~\eqref{eq:howtogetunif}
can also be applied to obtain statements with uniformities both with respect to
time $T \geq T_\varepsilon$ and to the strategies of opponent.

Because our aim involves the average payoffs $\ol{m}_T$ via the target function $\varphi$
as in $\cC_{\varphi(\ol{m}_T)}$, we require the uniform continuity of $\varphi$
for technical reasons, i.e., to carry over the negligible differences between the average payoffs and their estimation
in the approachability aim. (This assumption of uniform continuity can easily be dropped based on
the result of Theorem~\ref{th:main}; details are omitted.)

\paragraph{Conclusion: approachability aim.}
The decision-maker, enjoying a full monitoring, should construct a strategy such that
almost surely and uniformly over the opponent's strategies,
\begin{equation}
\label{eq:aim}
\dist_p \bigl( \ol{r}_T, \, \cC_{\varphi(\ol{m}_T)} \bigr)
\longrightarrow 0 \qquad \mbox{as} \ \ T \to \infty\,;
\end{equation}
that is, for all $\varepsilon > 0$, there exists $T_\varepsilon > 0$ such that
for all strategies of the opponent, with probability at least $1-\epsilon$,
\[
\sup_{T \geq T_\varepsilon} \, \dist_p \bigl( \ol{r}_T, \, \cC_{\varphi(\ol{m}_T)} \bigr)
\, \leq \varepsilon\,.
\]
We note that we will often be able to provide stronger, uniform and deterministic controls, of the form:
there exists a function $\zeta : \{1,2,3,\ldots\} \to [0,+\infty)$ such that
$\zeta(T) \to 0$ and for all strategies of the opponent,
\[
\dist_p \bigl( \ol{r}_T, \, \cC_{\varphi(\ol{m}_T)} \bigr) \leq \zeta(T)\,.
\]
To conclude this section,
we point out again that the two relaxations considered come at no cost in the generality
of setting: they are only intended to simplify and clarify the exposition.
Full details of this standard reduction from the case of bandit monitoring to
full monitoring are omitted because they are classical, though lengthy and technical, to expose.

\subsection{Link with approachability in known finite games}
\label{sec:setting2}

We link here our general setting above with the
classical setting considered by \citet{Bla56}.
Therein the decision-maker and the opponent have finite sets of actions $\cA$ and $\cB$,
and choose at each round $t$ respective pure actions $a_t \in \cA$ and $b_t \in \cB$,
possibly at random according to some mixed actions
$x_t = (x_{t,a})_{a \in \cA} \in \Delta(\cA)$ and $y_t = (y_{t,b})_{b \in \cB} \in \Delta(\cB)$.
A payoff function $r : \cA \times \cB \to \R^d$ is given and is
multilinearly extended to $\Delta(\cA) \times \Delta(\cB)$ according to
\[
\forall \, (x,y) \in \Delta(\cA) \times \Delta(\cB),
\qquad r(x,y)=\sum_{a \in \cA} \sum_{b \in \cB} x_a \, y_b \, r(a,b)\, .
\]
From the decision-maker viewpoint, the game takes place as if the opponent
was choosing at each round the vector of vector payoffs
\[
m_t = r(\,\cdot\,,b_t) = \bigl( r(a,b_t) \bigr)_{a \in \cA}\,.
\]
A target set $\cC$ is to be approached, that is, the convergence
\[
\widetilde{r}_T = \frac{1}{T} \sum_{t=1}^T r(a_t,b_t) \, \longrightarrow \, \cC
\quad \mbox{a.s.}
\]
should hold uniformly over the opponent's strategies.
(Of course, as recalled above, we can equivalently require the uniform convergence of $\ol{r}_T$
to $\cC$.)

A necessary and sufficient condition for
this when $\cC$ is
closed and convex is that for all $y \in \Delta(\cB)$, there exists some
$x \in \Delta(\cA)$ such that $r(x,y) \in \cC$. Of course, this condition, called the dual condition for
approachability, is not always met. However, in view of the dual condition,
the least approachable $\alpha$--expansion in $\ell_p$--norm of such a non-empty, closed, and convex set $\cC$
is given by
\begin{equation}
\label{eq:alphaunif}
\alpha_{\unif} = \max_{y \in \Delta(\cB)} \, \min_{x \in \Delta(\cA)}  \, \dist_p\bigl(r(x,y),\,\cC\bigr)\,.
\end{equation}
Approaching $\cC_{\alpha_{\unif}}$ corresponds to considering the constant target function
$\varphi \equiv \alpha_{\unif}$ in~\eqref{eq:aim}.
Better (uniformly smaller) choices of target functions exist,
as will be discussed in Section~\ref{sec:opportunistic}. This will be put in
correspondence therein with what is called ``opportunistic approachability.''

\paragraph{The knowledge of $r$ is crucial (a first strategy).}
The general strategies used to approach $\cC$ (or $\cC_{\alpha_{\unif}}$ when $\cC$ is not approachable
and $p=2$)
rely crucially on the knowledge of $r$.

Indeed, the original strategy of \citet{Bla56} proceeds as follows:
at round $t+1$,
it first computes the projection $\widetilde{c}_t$
of $\widetilde{r}_t$ onto $\cC$. Then it picks $a_{t+1}$ at random
according to a mixed action $x_{t+1}$ such that
\begin{equation}
\label{eq:howtoappr}
\forall \, y \in \Delta(\cB), \qquad \bigl\langle
\widetilde{r}_t - \widetilde{c}_t, \,\,
r(x_{t+1},y) - \widetilde{c}_t \bigr\rangle \leq 0\,.
\end{equation}
When $\cC$ is approachable, such a mixed action always exists; 
 one can take, for instance,
\[
x_{t+1} \in \argmin_{x \in \Delta(\cA)} \max_{y \in \Delta(\cB)}
\bigl\langle \widetilde{r}_t - \widetilde{c}_t, \,\, r(x,y) \bigr\rangle\,.
\]
In general, the strategy thus heavily depends on the knowledge of $r$.

When $\cC$ is not approachable and $p=2$, the set $\cC_{\alpha_{\unif}}$ is the target and the
choice right above is still suitable to approach $\cC_{\alpha_{\unif}}$ in $\ell_2$--norm.
Indeed, the projection $\widetilde{d}_t$ of
$\widetilde{r}_t$ onto $\cC_{\alpha_{\unif}}$ is such that
$\widetilde{r}_t - \widetilde{d}_t$ is proportional to $\widetilde{r}_t - \widetilde{c}_t$,
thus
\[
\argmin_{x \in \Delta(\cA)} \max_{y \in \Delta(\cB)}
\bigl\langle \widetilde{r}_t - \widetilde{c}_t, \,\, r(x,y) \bigr\rangle
= \argmin_{x \in \Delta(\cA)} \max_{y \in \Delta(\cB)}
\bigl\langle \widetilde{r}_t - \widetilde{d}_t, \,\, r(x,y) \bigr\rangle\,.
\]

\paragraph{The knowledge of $r$ is crucial (a second strategy).}
There are other strategies to perform approachability in known finite games,
though the one described above may be the most popular one.
For instance, \citet{BernShim} propose a strategy based on the dual
condition for approachability, which still performs approachability at the
optimal $1/\sqrt{T}$ rate. We discuss it in greater details and
generalize it to the case of unknown games in Section~\ref{sec:disc}.
For now, we describe it shortly only to show how heavily it relies on
the game $r$ being known. Assume that $\cC$ is approachable.
At round $t=1$, choose an arbitrary mixed action $x_1$ to draw $a_1$
and choose an arbitrary mixed action $\ty_1 \in \cB$. For rounds $t+1$,
assume that mixed actions $\ty_1,\ldots,\ty_t \in \Delta(\cB)$
have been chosen by the decision-maker in addition to the pure actions $b_1,\ldots,b_t$
actually played by the opponent, and that corresponding mixed actions $\tx_1,\ldots,\tx_t$
such that $r\bigl(\tx_s,\ty_s\bigr) \in \cC$ have been chosen as well. Denoting
\[
\delta_t = \sum_{\tau=1}^t r(a_\tau,b_\tau) - \sum_{\tau=1}^t r\bigl(\tx_\tau,\ty_\tau\bigr)\,,
\]
the strategy selects
\[
x_{t+1} \in \argmin_{x \in \Delta(\cA)} \max_{y \in \Delta(\cB)} \langle \delta_t, \,\, r(x,y) \rangle
\qquad \mbox{and} \qquad
\ty_{t+1} \in \argmax_{y \in \Delta(\cB)} \min_{x \in \Delta(\cA)} \langle \delta_t, \,\, r(x,y) \rangle\,,
\]
as well as $\tx_{t+1} \in \Delta(\cA)$ such that $r\bigl(\tx_{t+1},\ty_{t+1}\bigr) \in \cC$,
where such an $\tx_{t+1}$ exists since $\cC$ is approachable.
Thus, it is crucial that the strategy knows $r$; however, that $\cC$ be approachable
is not essential: in case it is not approachable and $\cC_{\alpha_{\unif}}$ is to be approached
instead, it suffices to pick
\[
\tx_{s} \in \argmin_{x \in \Delta(\cA)} \dist_p\Bigl(r\bigl(x,\ty_s\bigr),\,\cC\Bigr)\,,
\]
so that $r\bigl(\tx_s,\ty_s\bigr) \in \cC_{\alpha_{\unif}}$. Any $\ell_p$--norm is suitable
for this argument.

\subsection{Link with regret minimization in unknown games}
\label{sec:regretunknown}

The problem of regret minimization can be encompassed as an instance of approachability.
For the sake of completeness, we recall in Appendix~\ref{sec:app:regretunknown} why the knowledge of the payoff
structure is not crucial for this very specific problem. This, of course, is not the case at all
for general approachability problems.

\section{Two toy examples to develop some intuition}
\label{sec:descrex}

The examples presented below will serve as guides to
determine suitable target functions $\varphi : \cK \to [0,+\infty)$,
that is, target functions for which the convergence~\eqref{eq:aim}
can be guaranteed and that are ambitious (small) enough, in a sense that will be made formal
in the next section.

\paragraph{Example 1: minimize several costs at a time.}
The following example is a toy modeling of a case when the first player has to perform several
tasks simultaneously and incurs a loss (or a cost) for each of them; we assume that her overall loss is
the worst (the largest) of the losses thus suffered.

For simplicity, and because it will be enough for our purpose,
we will assume that the decision-maker only has two actions, that is, $\cA = \{ 1,2 \}$,
while the opponent is restricted to only pick convex combinations of the following vectors of vector payoffs:
\begin{align*}
& m^{\dag} = \big(m^{\dag}_a\big)_{a \in \{1,2\}} \in \R^{2 \times 2} \qquad
\mbox{with} \quad m^{\dag}_1 = \left[\begin{array}{c} 3 \\ 4 \end{array}\right] \quad
\mbox{and} \quad m^{\dag}_2 = \left[\begin{array}{c} 0 \\ 5 \end{array}\right], \\
& \phantom{r} \\
\mbox{and} \qquad\qquad &
m^{\sharp} = \big(m^{\sharp}_a\big)_{a \in \{1,2\}} \in \R^{2 \times 2} \qquad
\mbox{with} \quad m^{\sharp}_1 = \left[\begin{array}{c} 4 \\ 3 \end{array}\right] \quad
\mbox{and} \quad m^{\sharp}_2 = \left[\begin{array}{c} 5 \\ 0 \end{array}\right].
\end{align*}
The opponent's actions can thus be indexed by $\nu \in [0,1]$,
where the latter corresponds to
the vector of vectors
\[
m^{(\nu)} = \nu m^{\dag} + (1-\nu)m^{\sharp} \,\, \in \R^{2 \times 2}\,.
\]

The base target set $\cC$ is the negative orthant $\cC = (-\infty,0]^2$ and
its $\alpha$--expansions in the supremum norm ($p = +\infty$)
are $\cC_\alpha = (-\infty,\alpha]^2$. A graphical representation
of these expansions and of the vectors $m^{\dag}$ and $m^{\sharp}$
is provided in Figure~\ref{FG:Fig1-1}.

\begin{figure}[h!]
\begin{center}
\includegraphics{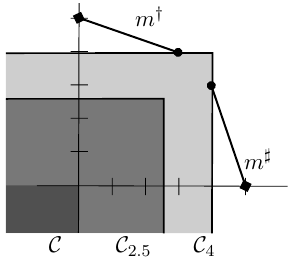}
\caption{\label{FG:Fig1-1} Graphical representation of $m^\dag$ and $m^\sharp$
and of different expansions $\cC_\alpha$.}
\end{center}
\end{figure}

\paragraph{Example 2: control absolute values.} In this example, the decision-maker still has only
two actions, $\cA = \{1,2\}$, and gets scalar rewards, i.e., $d = 1$. The aim is to minimize the absolute value
of the average payoff, i.e., to control the latter from above and from below (for instance, because these
payoffs measure deviations in either direction from a desired situation).

Formally, the opponent chooses vectors $m_t \in \R^2$, which we assume to actually lie in $\cK = [-1,1]^2$.
The product $\odot$ is then simply the standard inner product over $\R^2$. We
consider $\cC = \{0\}$ as a base target set to be approached. Its expansions (in any $\ell_p$--norm) are
$\cC = [-\alpha,\alpha]$, for $\alpha \geq 0$.

\subsection{The smallest set in hindsight cannot be achieved in general}
\label{sec:varphistar}

We denote by $\varphi^\star : \cK \to [0,+\infty)$ the function that associates
with a vector of vector payoffs $m \in \cK$
the index of the smallest $\ell_p$--expansion of $\cC$ containing a convex combination of its components:
\begin{equation}
\label{eq:phistar}
\varphi^\star(m) = \min \Big\{ \alpha \geq 0 : \ \ \exists \, x \in \Delta(\cA) \ \mbox{s.t.} \
x \odot m \in \cC_\alpha \Bigr\} = \min_{x \in \Delta(\cA)} \dist_p(x \odot m,\,\cC)\,,
\end{equation}
the infimum being achieved by continuity. This defines a function $x^\star$:
\[
\forall \, m \in \cK, \qquad \exists \, x^\star(m) \in \Delta(\cA) : \qquad
x^\star(m) \odot m \in \cC_{\varphi^\star(m)}\,.
\]

\begin{lemma}
\label{lm:notach}
In Examples~1 and~2,
the convergence~\eqref{eq:aim} cannot be achieved for $\varphi^\star$ against
all strategies of the opponent.
\end{lemma}

The proofs (located in Appendix~\ref{sec:Expost-notach})
reveal that the difficulty in~\eqref{eq:aim} is that it should
hold along a whole path, while the value of $\varphi^\star\bigl( \ol{m}_T \bigr)$
can change more rapidly than the average payoff vectors $\ol{r}_T$ do.

They will formalize the following proof scheme.
To accommodate a first situation, which lasts a large number $T'$
of stages, the decision-maker should play in a given way;
but then, the opponent changes drastically his strategy and from where the decision-maker
is she cannot catch up and is far from the target at stage $2T'$. The situation is repeated.

\subsection{A concave relaxation is not ambitious enough}
\label{sec:cav}
\label{sec:ex}

A classical relaxation in the game-theory literature for unachievable
targets (see, e.g., how \citealp{Pathwise} proceed)
is to consider concavifications.
Can the convergence~\eqref{eq:aim} hold with $\varphi = \cav[\varphi^\star]$,
the concavification of $\varphi^\star$? The latter is
defined as the least concave function $\cK \to [0,+\infty]$ above $\varphi^\star$.
The next section will show that it is indeed always the
case but we illustrate on our examples why such a goal is not ambitious enough.
(The proof of the lemma below can be found in Appendix~\ref{sec:cavnotamb}.)

\begin{lemma}
\label{lm:cavnotamb}
In Examples~1 and~2,
the decision-maker has a mixed action $(x,1-x)$
that she can play at each round to ensure the convergence~\eqref{eq:aim}
for a target function $\varphi_x$ that is uniformly smaller than $\varphi^\star$,
and even strictly smaller at some points.
\end{lemma}

\section{A general class of ambitious enough target functions}
\label{sec:choosetarget}

The previous section showed on examples that the best-in-hindsight target function $\varphi^\star$
was too ambitious a goal
 while its concavification $\cav[\varphi^\star]$ seemed not ambitious enough.
In this section, based on the intuition given by the formula for concavification, we provide a whole class
of achievable target functions, relying on a parameter: a response function $\Psi$.

In the definition below, by uniformity over strategies of the opponent player, we mean the uniform convergence stated
right after~\eqref{eq:aim}.
We denote by $\cG_\varphi$ the graph of the set-valued mapping $m \in \cK \mapsto \cC_{\varphi(m)}$:
\[
\cG_\varphi = \Bigr\{ (m,r) \in \cK \times \R^d \ \ \mbox{\rm s.t.} \ \ r \in \cC_{\varphi(m)} \Bigr\}
\subset (\R^{d})^{\cA} \times \R^d\,.
\]

\begin{definition}
\label{def:defach}
A continuous target function $\varphi : \cK \to [0,+\infty)$ is achievable if the decision-maker has a strategy ensuring
that, uniformly over all strategies of the opponent player,
\begin{equation}
\label{eq:cv}
\dist_p \bigl( \ol{r}_T, \, \cC_{\varphi(\ol{m}_T)} \bigr) \longrightarrow 0 \qquad \mbox{as} \ \ T \to \infty\,.
\end{equation}
More generally, a (possibly non-continuous) target function $\varphi : \cK \to [0,+\infty)$ is achievable if
$\cG_\varphi$ is approachable for the game with payoff function $(x,m)
\in \Delta(\cA) \times \cK \mapsto (m, \, x \odot m)$, that is,
if uniformly over all strategies of the opponent player,
\begin{equation}
\label{eq:cvgal}
(\ol{m}_T,\,\ol{r}_T) \longrightarrow
\cG_\varphi \qquad \mbox{as} \ \ T \to \infty\,.
\end{equation}
\end{definition}

We always have that~\eqref{eq:cv} entails \eqref{eq:cvgal}, with or without continuity of $\varphi$.
The condition~\eqref{eq:cvgal} is however less restrictive in general
and it is useful in the case of non-continuous target functions (e.g., to avoid lack of convergence
due to errors at early stages).
But for continuous target functions $\varphi$, the two definitions~\eqref{eq:cv} and~\eqref{eq:cvgal}
are equivalent. We prove these two facts in Section~\ref{sec:B1} in the appendix.

The defining equalities~\eqref{eq:phistar} for $\varphi^\star$ show that this function is continuous
(it is even a Lipschitz function
with constant $1$ in the $\ell_p$--norm). We already showed in Section~\ref{sec:varphistar} that
the target function $\varphi^\star$ is not achievable in general.

To be able to compare target functions, we consider the following definition and notation.

\begin{definition}
\label{def:prec}
A target function $\varphi : \cK \to [0,+\infty)$ is strictly smaller than
another target function $\varphi'$ if $\varphi \leq \varphi'$
and there exists $m \in \cK$ with $\varphi(m) < \varphi'(m)$.
We denote this fact by $\varphi \prec \varphi'$.
\end{definition}

For instance, in Lemma~\ref{lm:cavnotamb}, we had $\varphi_x \prec \varphi^\star$.

\subsection{The target function $\cav[\varphi^\star]$ is always achievable}

We show below that the target function $\cav[\varphi^\star]$ is always achievable...
But of course, Section~\ref{sec:cav} already showed that $\cav[\varphi^\star]$ is not ambitious enough:
in Examples~1 and~2, there exist easy-to-construct
achievable target functions $\varphi$ with $\varphi \prec \cav[\varphi^\star]$.
We however provide here a general study of the achievability of $\cav[\varphi^\star]$ as it sheds light on how to achieve
more ambitious target functions.

So, we now only ask for convergence of $(\ol{m}_T,\,\ol{r}_T)$ to the convex hull of $\cG_{\varphi^\star}$,
not to $\cG_{\varphi^\star}$ itself. Indeed,
this convex hull is exactly the graph $\cG_{\cav[\varphi^\star]}$, where $\cav[\varphi^\star]$ is the
concavification of $\varphi^\star$, defined as the least concave function $\cK \to [0,+\infty]$
above $\varphi^\star$. Its variational
expression reads
\begin{equation}
\label{eq:cavphi}
\cav[\varphi^\star](m) =
\sup \! \left\{ \sum_{i \leq N} \lambda_i \, \varphi^\star(m_i) : \ \
N \geq 1 \ \ \mbox{and} \ \ \sum_{i \leq N} \lambda_i m_i = m \right\},
\end{equation}
for all $m \in \cK$, where the supremum is over all finite
convex decompositions of $m$ as elements of $\cK$
(i.e., the $m_i$ belong to $\cK$ and the $\lambda_i$ factors are nonnegative and sum up to $1$).
By a theorem of Fenchel and Bunt (see~\citealp[Theorem~1.3.7]{ConvexAna})
we could actually further impose that $1 \leq N \leq dA+1$.
In general, $\cav[\varphi^\star]$ is not continuous; it is however so when, e.g., $\cK$ is a polytope.

\begin{lemma}
\label{lm:ach}
The target function $\cav[\varphi^\star]$ is always achievable.
\end{lemma}

\proofof{sketch; when $\cK$ is known}
When the decision-maker knows $\cK$ (and only in this case), she can compute $\cav[\varphi^\star]$ and its graph
$\cG_{\cav[\varphi^\star]}$. As indicated after Definition~\ref{def:defach}, it suffices to
show that the convex set $\cG_{\cav[\varphi^\star]}$ is approachable
for the game with payoffs $(x,m) \in \Delta(\cA) \times \cK \mapsto (m, \, x \odot m)$; the decision-maker
should then play any strategy approaching $\cG_{\cav[\varphi^\star]}$.
Note that $\varphi^\star$ is continuous, that $\cG_{\varphi^\star}$ is thus a closed set,
and that $\cG_{\cav[\varphi^\star]}$ is a closed convex set containing $\cG_{\varphi^\star}$.
Now, the characterization of approachability by~\citet{Bla56} for closed convex sets (recalled
already in Section~\ref{sec:setting2}) states that for all $m \in \cK$, there should exist $x \in \Delta(\cA)$
such that $(m, \, x \odot m) \in \cG_{\cav[\varphi^\star]}$. But by the definition~\eqref{eq:phistar},
we even have $\bigl(m, \, x^\star(m) \odot m \bigr) \in \cG_{\varphi^\star}$, which concludes the proof.
\end{proof}

We only proved Lemma~\ref{lm:ach} under the assumption that the decision-maker knows $\cK$,
a restriction which we are however not ready to consider as indicated in Remark~\ref{rk:K}. Indeed, she
needs to know $\cK$ to compute $\cG_{\varphi^\star}$ and the needed projections onto this set to implement
Blackwell's approachability strategy. Some other approachability strategies may not require this knowledge,
e.g., a generalized version of the one of \citet{BernShim} based on the dual condition for approachability
(see Section~\ref{sec:setting2} for their original version, see Section~\ref{sec:BS-gener} for our generalization).

But anyway, we chose not to go into these details now because at least in the case when $\cC$ is convex,
Lemma~\ref{lm:ach} will anyway
follow from Lemmas~\ref{lm:xstar} and~\ref{lm:phixstarnotopti} (or Theorem~\ref{th:main}) below,
which are proved independently and wherein no knowledge of $\cK$ is assumed\footnote{Indeed,
the functions $x^\star$ and $\varphi^\star$ at hand therein
are independent of $\cK$ as they are defined for each $m \in \cK$ as the solutions of some
optimization program that only depends on this specific $m$ and on $\cC$, but not on $\cK$.}.
Even better, they prove the strongest notion of convergence~\eqref{eq:cv} of Definition~\ref{def:defach},
irrespectively of the continuity or lack of continuity of $\cav[\varphi^\star]$.

\subsection{An example of a more ambitious target function}
\label{sec:moreambitious}

By~\eqref{eq:phistar} we can rewrite~\eqref{eq:cavphi} as
\[
\cav[\varphi^\star](m) = \sup \left\{ \sum_{i \leq N} \lambda_i \, \dist_p\big(x^\star(m_i) \odot m_i,\,\cC\big) : \ \
N \geq 1 \ \ \mbox{and} \ \
\sum_{i \leq N} \lambda_i m_i = m \right\} .
\]
Now, whenever $\cC$ is convex, the function $\dist_p(\,\cdot\,,\,\cC\big)$ is convex as well over $\R^d$;
see, e.g., \citet[Example~3.16]{BoydV}. Therefore, denoting by $\varphi^{x^\star}$ the function defined as
\begin{equation}
\label{eq:phixstar}
\varphi^{x^\star}(m) = \sup \left\{ \dist_p\!\Biggl( \sum_{i \leq N} \lambda_i \, x^\star(m_i) \odot m_i,\,\,\cC \Biggr) : \ \
N \geq 1 \ \ \mbox{and} \ \
\sum_{i \leq N} \lambda_i m_i = m \right\}
\end{equation}
for all $m \in \cK$, we have $\varphi^{x^\star} \leq \cav[\varphi^\star]$.
The two examples considered in Section~\ref{sec:ex} actually show that
this inequality can be strict at some points. We summarize these facts in the lemma below,
whose proof can be found in Appendix~\ref{sec:lm:xstar}.
That $\varphi^{x^\star}$ is achievable is a special case of Lemma~\ref{lm:phixstarnotopti}
stated in the next subsection, where a class generalizing the form of $\varphi^{x^\star}$
will be discussed.

\begin{lemma}
\label{lm:xstar}
The inequality $\varphi^{x^\star} \leq \cav[\varphi^\star]$ always holds when $\cC$ is convex.
For Examples~1 and~2, we even have $\varphi^{x^\star} \prec \cav[\varphi^\star]$.
\end{lemma}


\subsection{A general class of achievable target functions}
\label{sec:truegoal}

The class is formulated by generalizing the definition~\eqref{eq:phixstar}:
we call response function any function $\Psi : \cK \to \Delta(\cA)$ and
we replace in~\eqref{eq:phixstar} the specific response function $x^\star$ by any response function~$\Psi$.

\begin{definition}
The target function $\varphi^\Psi$ based on the response function~$\Psi$ is defined, for all $m \in \cK$, as
\begin{equation}
\label{eq:phiPsi}
\varphi^\Psi(m) = \sup \left\{ \dist_p\!\left(\sum_{i \leq N} \lambda_i \, \Psi(m_i) \odot m_i,\,\,\cC\right) : \ \
N \geq 1 \ \ \mbox{and}
\ \ \sum_{i \leq N} \lambda_i m_i = m \right\}.
\end{equation}
\end{definition}

\begin{lemma}
\label{lm:phixstarnotopti}
For all response functions $\Psi$, the target functions $\varphi^\Psi$ are achievable.
\end{lemma}

The lemma actually follows from
Theorem~\ref{th:main} below, which provides an explicit and efficient strategy to achieve
any $\varphi^\Psi$, in the stronger sense~\eqref{eq:cv} irrespectively of the continuity or lack of
continuity of $\varphi^\Psi$.
For now, we provide a sketch of proof (under an additional assumption
of Lipschitzness for $\Psi$) based on calibration, because it
further explains the intuition behind~\eqref{eq:phiPsi}.
It also advocates why the $\varphi^\Psi$ functions are reasonable targets:
resorting to some auxiliary calibrated strategy outputting accurate predictions $\hat{m}_t$ (in the sense
of calibration) of the vectors $m_t$ almost amounts to knowing in advance the $m_t$. And with
such a knowledge, what can we get? \\

\proofof{sketch; when $\Psi$ is a Lipschitz function}
We will show below that there exists a constant $L'$ ensuring the following:
given any $\delta > 0$, there exists randomized strategy of the decision-maker
such that for all $\varepsilon > 0$, there exists a time $T_\varepsilon$ such that
for all strategies of the opponent, with probability at least $1-\varepsilon$,
\begin{equation}
\label{eq:calibrgoal}
\sup_{T \geq T_\varepsilon} \,\, \dist_p \bigl( \ol{r}_T, \, \cC_{\varphi^\Psi(\ol{m}_T)} \bigr) \leq
L' (\delta + \varepsilon) \qquad \mbox{a.s.}
\end{equation}
In terms of approachability theory (see, e.g., \citealp{Per} for a survey),
this means that $\cG_{\varphi^\Psi}$ is in particular an $L'\delta$--approachable set for
all $\delta >0$, thus a $0$--approachable set. But $0$--approachability and approachability are
two equivalents notions (a not-so-trivial fact when the sets at hand are not closed convex sets).
That is, $\cG_{\varphi^\Psi}$ is approachable, or put differently, $\varphi^\Psi$ is achievable.

Indeed, fixing $\delta > 0$, there exists a randomized strategy picking
predictions $\widehat{m}_t$ among finitely many
elements $m^{[j]} \in \cK$, where $j \in \{1,\ldots,N_\delta\}$ so that
the so-called calibration score is controlled:
for all $\varepsilon > 0$, there exists a time $T_\varepsilon$ such that
for all strategies of the opponent, with probability at least $1-\varepsilon$,
\begin{equation}
\label{eq:calibrscore}
\sup_{T \geq T_\varepsilon} \,\, \sum_{j=1}^{N_\delta} \left\Arrowvert \frac{1}{T} \sum_{t=1}^T
\ind_{ \bigl\{ \widehat{m}_t = m^{[j]} \bigr\} } \bigl( \widehat{m}_t - m_t \bigr) \right\Arrowvert_p
\ \ \leq \delta + \varepsilon \qquad \mbox{a.s.};
\end{equation}
see\footnote{Actually, the latter reference only considers the case of calibrated predictions
of elements in some simplex, but it is clear from the method used in~\citet{MaSt10} ---
a reduction to a problem of approachability ---
that this can be performed for all subsets of compact sets, such as $\cK$ here,
with the desired uniformity over the opponent's strategies; see
also \citet[Appendix~B]{JMLR14}. The result holds
for any $\ell_p$--norm by equivalence of norms on vector spaces of finite dimension,
even if the original references considered the $\ell_1$ or $\ell_2$--norms only.}~\citet{FoVo98}.
Now, the main strategy, based on such an auxiliary calibrated strategy, is to play $\Psi(\widehat{m}_t)$
at each round. The average payoff of the decision-maker is thus
\[
\ol{r}_T = \frac{1}{T} \sum_{t=1}^T \Psi\bigl( \widehat{m}_t \bigr) \odot m_t\,.
\]
We decompose it depending on the predictions $\widehat{m}_t$ made:
for each $j \in \{1,\ldots,N_\delta\}$, the average number of
times $m^{[j]}$ was predicted and the average vectors of vector payoffs obtained
on the corresponding rounds equal
\[
\widehat{\lambda}_{j,T} = \frac{1}{T} \sum_{t=1}^T \ind_{ \bigl\{ \widehat{m}_t = m^{[j]} \bigr\} }
\qquad \mbox{and} \qquad
\ol{\widehat{m}}^{j,T} = \frac{\sum_{t=1}^T m_t \, \ind_{ \bigl\{ \widehat{m}_t = m^{[j]} \bigr\} }}{
\sum_{t=1}^T \ind_{ \bigl\{ \widehat{m}_t = m^{[j]} \bigr\} }}
\]
whenever $\widehat{\lambda}_{j,T} > 0$, otherwise, we take an arbitrary value for
$\ol{\widehat{m}}^{j,T}$. In particular,
\[
\ol{m}_T = \sum_{j=1}^{N_\delta} \widehat{\lambda}_{j,T} \, \ol{\widehat{m}}^{j,T}
\qquad \mbox{and} \qquad
\ol{r}_T = \sum_{j=1}^{N_\delta} \widehat{\lambda}_{j,T} \,
\Psi\bigl( m^{[j]} \bigr) \odot \ol{\widehat{m}}^{j,T}\,.
\]
Using this convex decomposition of $\ol{m}_T$ in terms of elements of $\cK$,
the very definition of $\varphi^\Psi$ leads to
\[
\sum_{j=1}^{N_\delta} \widehat{\lambda}_{j,T} \,\,
\Psi\Bigl( \ol{\widehat{m}}^{j,T} \Bigr) \odot \ol{\widehat{m}}^{j,T}
\,\in\, \cC_{\varphi^\Psi(\ol{m}_T)}\,,
\]
hence
\begin{multline*}
\dist_p \bigl( \ol{r}_T, \, \cC_{\varphi^\Psi(\ol{m}_T)} \bigr)
\leq \left\Arrowvert \ol{r}_T -
\sum_{j=1}^{N_\delta} \widehat{\lambda}_{j,T} \,
\Psi\Bigl( \ol{\widehat{m}}^{j,T} \Bigr) \odot \ol{\widehat{m}}^{j,T}
\right\Arrowvert_p \\
= \left\Arrowvert
\sum_{j=1}^{N_\delta} \widehat{\lambda}_{j,T}
\Bigl( \Psi\bigl( m^{[j]} \bigr)
-
\Psi\Bigl( \ol{\widehat{m}}^{j,T} \Bigr) \Bigr)
\odot \ol{\widehat{m}}^{j,T}
\right\Arrowvert_p\,.
\end{multline*}
We denote by $B_{p,\max}$ a bound on the
maximal $\ell_p$--norm of an element in the bounded set $\cK$. A triangular
equality shows that
\begin{align*}
\left\Arrowvert
\sum_{j=1}^{N_\delta} \widehat{\lambda}_{j,T}
\Bigl( \Psi\bigl( m^{[j]} \bigr)
-
\Psi\Bigl( \ol{\widehat{m}}^{j,T} \Bigr) \Bigr)
\odot \ol{\widehat{m}}^{j,T}
\right\Arrowvert_p
& \leq \sum_{j=1}^{N_\delta}
\widehat{\lambda}_{j,T} \sum_{a \in \cA}
\biggl| \Psi\bigl( m^{[j]} \bigr)_a
-
\Psi\Bigl( \ol{\widehat{m}}^{j,T} \Bigr)_a \biggr|
\, \Bigl\Arrowvert \ol{\widehat{m}}^{j,T}_a \Bigr\Arrowvert_p \\
& \leq
B_{p,\max} \sum_{j=1}^{N_\delta} \widehat{\lambda}_{j,T} \sum_{a \in \cA}
\biggl| \Psi\bigl( m^{[j]} \bigr)_a
-
\Psi\Bigl( \ol{\widehat{m}}^{j,T} \Bigr)_a \biggr|
\end{align*}
where $\psi(m)_a$ refers to the probability mass put on $a \in \cA$ by $\psi(m)$.
As indicated above, we assume for this sketch of proof that $\Psi$ is a Lipschitz function,
with Lipschitz constant $L$ with respect to the $\ell_1$--norm over $\Delta(\cA)$
and the $\ell_p$--norm over $\cK$. We get
\begin{align*}
B_{p,\max} \sum_{j=1}^{N_\delta} \widehat{\lambda}_{j,T} \sum_{a \in \cA}
\biggl| \Psi\bigl( m^{[j]} \bigr)_a - \Psi\Bigl( \ol{\widehat{m}}^{j,T} \Bigr)_a \biggr|
& \leq B_{p,\max} \, L \sum_{j=1}^{N_\delta} \widehat{\lambda}_{j,T} \Bigl\Arrowvert
m^{[j]} - \ol{\widehat{m}}^{j,T} \Bigr\Arrowvert_p \\
& = B_{p,\max} \, L \sum_{j=1}^{N_\delta} \left\Arrowvert \frac{1}{T} \sum_{t=1}^T
\ind_{ \bigl\{ \widehat{m}_t = m^{[j]} \bigr\} } \bigl( \widehat{m}_t - m_t \bigr) \right\Arrowvert_p\,.
\end{align*}
Substituting~\eqref{eq:calibrscore}, we proved~\eqref{eq:calibrgoal} for
$L' = B_{p,\max} \, L$, which concludes the proof.
\end{proof}

\subsection{Some thoughts on the optimality of target functions}

The previous subsections showed that target functions of the form
$\varphi^\Psi$ were achievable, unlike the best-in-hindsight target
function $\varphi^\star$, and that they were more ambitious than the
concavification $\cav[\varphi^\star]$. The question
of their optimality can be raised --- a question to which we will not be able to answer
in general. Our thoughts are gathered in Appendix~\ref{sec:admiss}.

\section{A strategy by regret minimization in blocks}
\label{sec:regretmin}

In this section we exhibit a strategy to achieve the stronger notion of convergence~\eqref{eq:cv}
with the target functions $\varphi^\Psi$ advocated in Section~\ref{sec:truegoal},
irrespectively of the continuity or lack of continuity of $\varphi^\Psi$.
The algorithm is efficient, as long as calls to $\Psi$ are (a full discussion
of the complexity issues will be provided for each application studied in Section~\ref{sec:appli}).

\subsection{Description and analysis of the strategy}

As in~\citet{BNR}, the considered strategy --- see Figure~\ref{fig:strat-blocks} --- relies
on some auxiliary regret-minimizing strategy $\cR$, namely, a strategy
with the following property.

\begin{assumption}
\label{ass:R}
The strategy $\cR$ sequentially outputs mixed actions $u_t \in \Delta(\cA)$
such that for all ranges $B > 0$ (not necessarily known
in advance), for all $T \geq 1$ (not necessarily known in advance),
for all sequences of vectors $m'_t \in \R^{\cA}$ of one-dimensional payoffs
lying in the bounded interval $[-B,B]$,
possibly chosen online by some opponent player, where $t = 1,\ldots,T$,
\[
\max_{u \in \Delta(\cA)} \sum_{t=1}^T u \odot m'_t
\leq 4B\sqrt{T \ln A} + \sum_{t=1}^T u_t \odot m'_t\,.
\]
\end{assumption}

Note in particular that the auxiliary strategy $\cR$ automatically adapts to the range $B$ of the payoffs and to the number of rounds $T$,
and has a sublinear worst-case guarantee. (The adaptation to $B$ will be needed because $\cK$ is unknown.) Such auxiliary strategies indeed exist,
for instance, the polynomially weighted average forecaster of~\citet{CeLu03}.
Other ones with a possibly larger constant factor in front of the $B\sqrt{T \ln A}$ term also exist,
for instance, exponentially weighted average strategies with learning rates carefully tuned over time,
as described by~\citet{CeMaSt07} or~\citet{Rooijetal14}.

For the sake of elegance (but maybe at the cost of not providing all the intuitions that led us to this result),
we only provide in Figure~\ref{fig:strat-blocks}
the time-adaptive version of our strategy, which does not need to know the time horizon $T$ in advance.
The used blocks are of increasing lengths $1,\,2,\,3,\ldots$.
Simpler versions with fixed block length $L$ would require a tuning of $L$
in terms of $T$ (pick $L$ of the order of $\sqrt{T}$) to optimize
the theoretical bound.

\begin{figure}[t!h!]
\center
\begin{minipage}{0.96\textwidth}
\rule{\linewidth}{.5pt}{}
{\small
\emph{Parameters}: a regret-minimizing strategy $\cR$ (with initial action $u_1$) and
a response function $\Psi : \cK \to \Delta(\cA)$ \smallskip \\
\emph{Initialization}: play $x_1 = u_1$ and observe $m_1 \in (\R^{d})^{\cA}$;
this is block $n=1$ \smallskip \\
\emph{For} blocks $n = 2,3\ldots$,
\begin{enumerate}
\item compute the total discrepancy $\delta_n$ at the beginning\protect\footnote{
Block $n$ starts at round $\displaystyle{1+\frac{n(n-1)}{2}}$, is of length $n$,
thus lasts till round $\displaystyle{\frac{n(n+1)}{2}}$. \smallskip
}
of block $n$ (that is, till the end of block $n-1$),
\begin{align*}
& \delta_{n} = \sum_{t=1}^{n(n-1)/2} x_t \odot m_t - \sum_{k=1}^{n-1} k \, \Psi\bigl( \ol{m}^{(k)} \bigr) \odot \ol{m}^{(k)} \in \R^d\,, \\
\mbox{where} \qquad & \ol{m}^{(k)} = \frac{1}{k} \sum_{t = 1}^{k} m_{k(k-1)/2+t}
\end{align*}
is the average vector of vector payoffs obtained in block $k \in \{1,\ldots,n-1\}$;
\item run a fresh instance $\cR_n$ of $\cR$ for $n$ rounds as follows: set $u_{n,1} = u_1$;
then, for $t = 1,\ldots,n$,
\begin{enumerate}
\item play $x_{n(n-1)/2+t} = u_{n,t}$ and observe $m_{n(n-1)/2+t} \in (\R^{d})^{\cA}$;
\item feed $\cR_n$ with the vector payoff $m'_{n,t} \in \R^{\cA}$
with components given, for $a \in \cA$, by
\[
m'_{n,t,a} = - \langle \delta_{n}, \, m_{n(n-1)/2+t,a} \rangle \in \R,
\]
where $\langle\,\cdot\,,\,\cdot\,\rangle$ denotes the inner product in $\R^d$;
\item obtain from $\cR_n$ a mixed action $u_{n,t+1}$.
\end{enumerate}
\end{enumerate}
}
\end{minipage}
\rule{\linewidth}{.5pt}
\caption{\label{fig:strat-blocks} The proposed strategy, which plays in blocks of increasing lengths $1,\,2,\,3,\,\ldots$}
\end{figure}

\begin{theorem}
\label{th:main}
For all response functions $\Psi$,
the strategy of Figure~\ref{fig:strat-blocks} is such that
for all $T \geq 1$, for all strategies of the opponent,
there exists $c_T \in \cC_{\varphi^\Psi(\ol{m}_T)}$ ensuring
\begin{equation}
\label{eq:whattobound}
\bigl\Arrowvert \ol{r}_T - c_T \bigr\Arrowvert_2
\leq 10 \, T^{-1/4} \ln A + 3 B_{2,\max} T^{-1/2} \,,
\end{equation}
where $B_{2,\max} = \displaystyle{\max_{m \in \cK} \Arrowvert m \Arrowvert_2}$ is the maximal Euclidean norm of
elements in $\cK$.

In particular, denoting by $\kappa_p$ a constant such that $\Arrowvert \,\cdot\, \Arrowvert_p
\leq \kappa_p \, \Arrowvert \,\cdot\, \Arrowvert_2$,
for all $T \geq 1$ and all strategies of the opponent,
\begin{equation}
\label{eq:whattobound2}
\dist_p \bigl( \ol{r}_T, \, \cC_{\varphi^\Psi(\ol{m}_T)} \bigr)
\leq \kappa_p \bigl( 10 \, T^{-1/4} \ln A + 3 B_{2,\max} T^{-1/2} \bigr) \,.
\end{equation}
\end{theorem}

\begin{remark}
With the notation of Figure~\ref{fig:strat-blocks},
denoting in addition by $n_T$ the largest integer such that $n_T(n_T+1)/2 \leq T$,
by
\[
\ol{m}^{\parti} = \frac{1}{T - n_T(n_T+1)/2} \sum_{t = n_T(n_T+1)/2+1}^{T} m_t
\]
the partial average of the vectors of vector payoffs $m_t$ obtained during the last and $(n_T+1)$--th block
when $n_T(n_T+1)/2 < T$ (and an arbitrary element of $\cK$ otherwise), we can take
\begin{equation}
\label{eq:whattobound3}
c_T = \frac{1}{T} \left(
\sum_{k=1}^{n_T} k \, \Psi\bigl( \ol{m}^{(k)} \bigr) \odot \ol{m}^{(k)}
+ \left(T - \frac{n_T(n_T+1)}{2} \right)
\Psi\bigl(\ol{m}^{\parti}\bigr) \odot \ol{m}^{\parti}
\right).
\end{equation}
\end{remark}

\paragraph{Important comments on the result.}
The strategy itself does not rely on the knowledge of $\cK$,  as promised in Remark~\ref{rk:K};
only its performance bound does, via the $B_{2,\max}$ term.
Also, the convexity of $\cC$ is not required. The convergence rates are independent of the
ambient dimension $d$.

Concerning the norms, even if the strategy and its bound~\eqref{eq:whattobound} are based
on the Euclidean norm, the set $\cC_{\varphi^\Psi(\ol{m}_T)}$ is defined
in terms of the $\ell_p$--norm as in~\eqref{eq:phiPsi}. The constant $\kappa_p$ exists by
equivalence of the norms on a finite-dimensional space.

Finally, we note that we obtained the uniformity requirement stated after~\eqref{eq:aim}
in the deterministic form with a function $\zeta$ where $\zeta(T) = O\big(T^{-1/4}\big)$.
\medskip

\begin{proof}
The convergence~\eqref{eq:whattobound2} follows from the bound~\eqref{eq:whattobound}
via the equivalence between $\ell_p$-- and $\ell_2$--norms.
That the stated $c_T$ in~\eqref{eq:whattobound3} belongs to $\cC_{\varphi^\Psi(\ol{m}_T)}$,
where the latter set is defined in terms of the $\ell_p$--norm as in~\eqref{eq:phiPsi},
is by construction of $\varphi^\Psi$ as a supremum.
It thus suffices to prove~\eqref{eq:whattobound} with the $c_T$
defined in~\eqref{eq:whattobound3}, which we do by induction.

The induction is on the index $n \geq 1$ of the blocks, and the quantities to
control are the squared Euclidean norms of the discrepancies at the end
of these blocks, $\Arrowvert \delta_{n+1} \Arrowvert_2^2$. (We recall
that $\delta_{n+1}$ denotes the discrepancy at the end of block $n$.)
We have that $\delta_{2}$ is a difference between two elements of $\cK$,
thus that $\Arrowvert \delta_{2} \Arrowvert_2^2 \leq 4 B_{2,\max}^2$.

We use a self-confident approach: we consider a function $\beta : \{ 1,2,\ldots\} \to [0,+\infty)$ to be defined
by the analysis and assume that we have proved that our strategy is such that
for some $n \geq 1$ and for all sequences of vectors
of vector payoffs $m_t \in \cK$, possibly chosen by some opponent (i.e., for all
strategies of the opponent),
\[
\Arrowvert \delta_{n+1} \Arrowvert_2^2 =
\left\Arrowvert
\sum_{t=1}^{n(n+1)/2} x_t \odot m_t - \sum_{k=1}^{n} k \, \Psi\bigl( \ol{m}^{(k)} \bigr) \odot \ol{m}^{(k)}
\right\Arrowvert_2^2 \leq \beta(n)\,.
\]
For instance, we define $\beta(1) = 4 B_{2,\max}^2$.

We then study what we can guarantee for $n+2$. We have
\begin{eqnarray}
\nonumber
\Arrowvert \delta_{n+2} \Arrowvert_2^2 & = & \left\Arrowvert \delta_{n+1} + \left( \sum_{t=n(n+1)/2+1}^{(n+1)(n+2)/2} x_t \odot m_t -
(n+1) \, \Psi\bigl( \ol{m}^{(n+1)} \bigr) \odot \ol{m}^{(n+1)} \right) \right\Arrowvert_2^2 \\
\nonumber
& = & \Arrowvert \delta_{n+1} \Arrowvert_2^2 +
2 \left\langle \delta_{n+1}, \,\, \sum_{t=n(n+1)/2+1}^{(n+1)(n+2)/2} x_t \odot m_t -
(n+1) \, \Psi\bigl( \ol{m}^{(n+1)} \bigr) \odot \ol{m}^{(n+1)} \right\rangle \\
\label{eq:innprodthmain}
& & \quad + \left\Arrowvert \sum_{t=n(n+1)/2+1}^{(n+1)(n+2)/2} x_t \odot m_t -
(n+1) \, \Psi\bigl( \ol{m}^{(n+1)} \bigr) \odot \ol{m}^{(n+1)} \right\Arrowvert_2^2 \,.
\end{eqnarray}
We upper bound the two squared norms by $\beta(n)$ and $4 (n+1)^2 B_{2,\max}^2$, respectively.
Using the short-hand notation $u^{(n+1)} = \Psi\bigl( \ol{m}^{(n+1)} \bigr)$,
the inner product can be
rewritten, with the notation of Figure~\ref{fig:strat-blocks}, as
\begin{multline}
\label{eq:innerprod}
\left\langle \delta_{n+1}, \,\, \sum_{t=n(n+1)/2+1}^{(n+1)(n+2)/2} x_t \odot m_t -
(n+1) \, \Psi\bigl( \ol{m}^{(n+1)} \bigr) \odot \ol{m}^{(n+1)} \right\rangle \\
= - \sum_{t=1}^{n+1} u_{n+1,t} \odot m'_{n+1,t} + \sum_{t=1}^{n+1} u^{(n+1)} \odot m'_{n+1,t}
\end{multline}
Now, the Cauchy--Schwarz inequality indicates that for all $a$ and $t$,
\[
\bigl| m'_{n+1,t,a} \bigr| \leq \Arrowvert \delta_{n+1} \Arrowvert_2 \, \Arrowvert m_{n(n+1)/2+t,a} \Arrowvert_2
\leq B_{2,\max} \, \sqrt{\beta(n)}\,,
\]
where we used again the induction hypothesis.
Assumption~\ref{ass:R} therefore indicates that
the quantity~\eqref{eq:innerprod} can be bounded by $4 B_{2,\max} \sqrt{\beta(n)} \, \sqrt{(n+1) \ln A}$.

Putting everything together, we have proved that the induction holds provided that $\beta(n+1)$ is defined,
for instance, as
\[
\beta(n+1) = \beta(n) + 8 B_{2,\max} \sqrt{\beta(n)} \, \sqrt{(n+1) \ln A} + 4 B_{2,\max}^2 (n+1)^2\,.
\]
By the lemma in Appendix~\ref{sec:lemmaproofmainthm}
(taking $\gamma_1 = 4 B_{2,\max} \sqrt{\ln A}$ and $\gamma_2 = 4 B_{2,\max}^2$),
we thus get first
\[
\forall \, n \geq 1, \quad \beta(n) \leq 32 \, B_{2,\max}^2 (\ln A) \, n^3\,,
\qquad \mbox{then} \qquad
\forall \, n \geq 1, \quad \Arrowvert \delta_{n+1} \Arrowvert_2 \leq B_{2,\max} \sqrt{32 \, n^3 \ln A}\,.
\]

It only remains to relate the quantity at hand in~\eqref{eq:whattobound} and~\eqref{eq:whattobound3}
to the $\delta_{n+1}$.
By separating time till the end of the $n_T$--block and starting from the beginning of block $n_T+1$
(should the latter start strictly before $T$), we get
\begin{align*}
\ol{r}_T - c_T &
= \frac{1}{T} \sum_{t=1}^T x_t \odot m_t
-
\frac{1}{T} \left( \sum_{k=1}^{n_T} k \, \Psi\bigl( \ol{m}^{(k)} \bigr) \odot \ol{m}^{(k)}
+ \left(T - \frac{n_T(n_T+1)}{2} \right) \Psi\bigl(\ol{m}^{\parti}\bigr) \odot \ol{m}^{\parti}
\right) \\
& = \frac{1}{T} \delta_{n_T +1} + \frac{1}{T} \sum_{t = n_T(n_T+1)/2+1}^{T}
\Bigl( x_t - \Psi\bigl(\ol{m}^{\parti}\bigr) \Bigr) \odot m_t\,.
\end{align*}
The second sum contains at most $n_T$ elements, as the $(n_T+1)$-th regime is incomplete.
A triangular inequality thus shows that
\begin{align*}
\Arrowvert \ol{r}_T - c_T \Arrowvert_2 \leq \frac{1}{T} \, \Arrowvert \delta_{n_T +1} \Arrowvert_2
+ 2 B_{2,\max} \frac{n_T}{T} & \leq \frac{1}{T} \, B_{2,\max} \sqrt{32 \, n_T^3 \ln A}
+ 2 B_{2,\max} \frac{n_T}{T} \\
& \leq 10 \, T^{-1/4} \ln A + 3 B_{2,\max} T^{-1/2}\,,
\end{align*}
where we used the inequality $n_T(n_T+1)/2 \leq T$, its implication
$n_T \leq \sqrt{2T}$, as well as (for the sake of readability)
the bounds $\sqrt{32 \times 2^{3/2}} < 10$
and $2\sqrt{2} < 3$.
\end{proof}

\subsection{Discussion}
\label{sec:disc}

In this section we gather comments, remarks, and pointers to the literature.
We discuss in particular the links and improvements over the concurrent (and independent) works
by~\citet{BernShim} and~\citet{AFFT}.

\subsubsection{Do we have to play in blocks? Is the obtained $T^{-1/4}$ rate optimal?}

Our strategy proceeds in blocks, unlike the ones exhibited for the case of known games,
as the original strategy by~\citet{Bla56} or the more recent one by~\citet{BernShim},
see Section~\ref{sec:setting2}.
The calibration-based strategy considered in the proof of Lemma~\ref{lm:phixstarnotopti}
also performed some grouping, according to the finitely many possible values of the predicted vectors
of vector payoffs.
This is because the target set to approach is unknown: the decision-maker approaches a sequence of
expansions of this set, where the sizes $\varphi^\Psi(\ol{m}_T)$ of the expansions vary depending
on the sequence of realized averages $\ol{m}_T$ of vectors of vector payoffs.
When an approachable target set $\cC$ is given, the strategies, e.g.,
by~\citet{Bla56} or~\citet{BernShim}, do not need to perform any grouping.

Actually, it is easy to prove that the following quantity, which involves
no grouping in rounds, cannot be minimized in general:
\begin{align}
\label{eq:impossible}
& \Arrowvert \ol{r}_T - c'_T \Arrowvert_p =
\left\| \frac{1}{T}\sum_{t=1}^T x_t\odot m_t- \frac{1}{T}\sum_{t=1}^T \Psi(m_t)\odot m_t\right\|_p \\
\mbox{where} \qquad & c'_T = \frac{1}{T}\sum_{t=1}^T \Psi(m_t)\odot m_t\,.
\end{align}
Indeed, consider a toy case where the $m_t = (g_{t,a})_{a \in \cA}$ have scalar components $g_{t,a} \in \R$,
the negative orthant $\cC = (-\infty,0]$ is to be approached, whose expansions
are given by $\cC_{\alpha} = (-\infty,\alpha]$, for $\alpha \geq 0$.
Considering the response function $\Psi\bigl( (g_a)_{a\in\cA} \bigr) \in \argmax_{a \in \cA} g_a$,
we see that~\eqref{eq:impossible} boils down to controlling
\[
\left| \frac{1}{T} \sum_{t=1}^T \sum_{a \in \cA} x_{t,a} g_{t,a} -
\frac{1}{T} \sum_{t=1}^T \max_{a'_t \in \cA} g_{a'_t,t} \right|,
\]
which is impossible\footnote{\label{fn:LB} This can be seen, e.g., by taking $\cA = \{1,2\}$
and binary payoffs $g_{t,a} \in \{0,1\}$. The expectation of the per-round regret
is larger than a positive constant
when the $g_{t,a}$ are realizations of independent random variables $G_{t,a}$
identically distributed according to a symmetric Bernoulli distribution.
In particular, the regret is larger than this constant
for some sequence of binary payoffs $g_{t,a} \in \{0,1\}$.}.
This is in contrast with the regret~\eqref{eq:regretunknowndef}, which can be minimized.
The most severe issue here is not really the absolute value taken,
but the fact that we are comparing the decision-maker's payoff to the sum of the
instantaneous maxima of the payoffs $g_{a,t}$, instead of being interesting in the
maximum of their sums as in~\eqref{eq:regretunknowndef}.

So, the answer to the first question would be: yes, we have to play in blocks.
Given that, is the obtained $T^{-1/4}$ rate optimal? We can answer this question
in the positive by considering the same toy case as above. With this example,
the bound~\eqref{eq:whattobound} given the definition~\eqref{eq:whattobound3} of $c_T$
rewrites
\[
\left| \frac{1}{T} \sum_{t=1}^T \sum_{a \in \cA} x_{t,a} g_{t,a} -
\frac{1}{T} \left( \sum_{k=1}^{n_T} \max_{a'_k \in \cA} \sum_{t=k(k-1)/2+1}^{k(k+1)/2} g_{t,a'_k}
+ \max_{a'_{n_T+1} \in \cA} \sum_{t=n_T(n_T-1)/2+1}^{T} g_{t,a'_k}
\right) \right|,
\]
which corresponds to the control (from above and from below) of what is called
a (per-round) ``tracking regret'' for $n_T$ shifts. This notion was introduced by~\citet{HeWa98};
see also \citet[Chapter~5]{CeLu06} for a review of the results known for tracking regret.
In particular, the examples used therein to show the optimality of the bounds
(which are of the form of the one considered in Footnote~\ref{fn:LB}) can be adapted in our
context, so that the lower bound on tracking regret with $n_T$ shifts applies in our case: it is
of the order of $\sqrt{n_T/T}$, thus of $T^{-1/4}$.

In a nutshell, what we proved in these paragraphs is that if we are to ensure the convergence~\eqref{eq:aim}
by controlling a quantity of the form~\eqref{eq:whattobound} and~\eqref{eq:whattobound3},
then we have to proceed in blocks and convergence cannot hold at a faster rate than $T^{-1/4}$.
However, the associated strategy is computationally efficient.
Also, neither the convexity of $\cC$, nor the continuity of $\varphi^\Psi$ or of $\Psi$
are required, yet the stronger convergence~\eqref{eq:cv} is achieved, not only~\eqref{eq:cvgal}.

\subsubsection{Trading efficiency for a better rate; an interpretation of the different rates}

Theorem~\ref{th:main} shows that some set is approachable here, namely, the set $\cG_{\varphi^\Psi}$
defined in~\eqref{eq:cvgal}:
it is thus a B--set in the terminology of~\citet{Spi02}, see also~\citet{Hou71} as well as
a remark by~\citet{Bla56}. Therefore,
there exists some (abstract and possibly computationally extremely inefficient)
strategy which approaches it at a $1/\sqrt{T}$--rate.
Indeed, the proof of existence of such a strategy does not rely on any constructive argument.

Based on all remarks above, we may provide an intuitive interpretation
of the $T^{-1/4}$ rate obtained in Theorem~\ref{th:main}, versus the $1/\sqrt{T}$
rate achieved either in our context by the abstract strategy mentioned right above,
or associated with Blackwell's original strategy or variations of it
as the one by~\citet{BernShim}
in the classical case of known games and sets $\cC$ being known to be approachable.
The interpretation is in terms of the
number of significant (costly)
computational units $N_{\mbox{\tiny comp}}$ (projections, solutions of convex or linear programs, etc.)
to be performed. The strategies with the faster rate $1/\sqrt{T}$ perform at least one or two of these units at each round,
while our strategy does it only of the order of $\sqrt{T}$ times during $T$ rounds---they are encompassed
into the calls to $\Psi$ and take place at times $t = k(k-1)/2+1$ for $k \geq 1$.
In all these cases, the rate is proportional to $\sqrt{N_{\mbox{\tiny comp}}/T}$.

\subsubsection{On the related framework of~\citet{AFFT}}

The setting considered therein is exactly the one described in Section~\ref{sec:setting};
our works are concurrent and independent.
Crucial differences lie however in the aims pursued and in the nature of the results obtained.

The quality of a strategy is evaluated by~\citet{AFFT} based on some quasi-concave
and Lipschitz function $f : \R^d \to \R$.
With the notation of Theorem~\ref{th:main} (the straightforward extension to an unknown horizon $T$ of) their aim
is to guarantee that
\begin{equation}
\label{EQ:AFFT}
\liminf_{T \to \infty} \left\{ f\!\left(\frac{1}{T}\sum_{t=1}^T x_t\odot m_t \right)
- \min_{k\in\{1,\ldots,n_T-1\}} \max_{x \in \Delta(\cA)} f\bigl(x \odot \ol{m}^{(k)}\bigr) \right\} \geq 0\,,
\end{equation}
where we recall that $n_T$ is of order $\sqrt{T}$.
\citet{AFFT} mention that this convergence can take place at an optimal $T^{-1/4}$ rate.

Satisfying~\eqref{EQ:AFFT} and recovering this optimal rate is actually a direct consequence of our Theorem~\ref{th:main} and of the assumptions on $f$.
Indeed, \eqref{eq:whattobound} and~\eqref{eq:whattobound3} together with
the Lipschitz assumption on $f$ entail that
\begin{equation}
\label{EQ:AFFT2}
\liminf_{T \to \infty} \left\{
f\!\left(\frac{1}{T}\sum_{t=1}^T x_t\odot m_t\right) - f\!\left( O(1/\sqrt{T}) +
\frac{1}{T} \sum_{k=1}^{n_T-1} k\,\Psi(\ol{m}^{(k)}) \odot \ol{m}^{(k)}\right) \right\}
\geq 0\,.
\end{equation}
The quasi-concavity of $f$ implies that the image by $f$ of a convex combination is larger than the minimum of the images by $f$
of the convex combinations. Thus, \eqref{EQ:AFFT2} yields in particular
\[
\liminf_{T \to \infty} \left\{
f\!\left(\frac{1}{T}\sum_{t=1}^Tx_t\odot m_t\right) - \min_{k=1,\ldots,n_T-1}
\, f\!\bigl(\Psi(\ol{m}^{(k)}) \odot \ol{m}^{(k)}\bigr) \right\}
\geq 0\,.
\]
The convergence rate is the same as for~\eqref{EQ:AFFT2}, thus is of order at least $T^{-1/4}$.
Defining the response function $\Psi$ by
$\Psi(m) \in \displaystyle{\argmax_{x \in \Delta(\cA)} f(x\odot m)}$, we get~\eqref{EQ:AFFT}.

However, we need to underline that the aim~\eqref{EQ:AFFT} is extremely weak:
assume, for instance, that during some block Nature chooses $\ol{m}^{(k)}$ with identical
components such that
\[
\forall\, x \in \Delta(\cA), \qquad
f \bigl( x\odot \ol{m}^{(k)} \bigr) = \min f\,.
\]
Then \eqref{EQ:AFFT} is satisfied irrespectively of the algorithm.
On the contrary, the more demanding aim~\eqref{EQ:AFFT2} that we consider is not necessarily satisfied
and an appropriate algorithm---as our one---must be used.

In addition, the strategy designed by~\citet{AFFT} still requires some knowledge---the set $\cK$
of vectors of vector payoffs needs to be known (which is a severe restriction)---and uses projections onto convex sets.
The rate they obtain for their weaker aim is $O(T^{-1/4})$, as we get for our improved aim.

\subsubsection{Links with the strategy of~\citet{BernShim}}
\label{sec:BS-gener}

In this final paragraph of our discussion of Theorem~\ref{th:main}
we review the strategy of~\citet{BernShim} and extend it, as much as it can
be extended, to a setting as close as possible to our setting of unknown games:
see Figure~\ref{fig:BSstrat}.
The extension however requires that the set $\cK$
of possible vectors of vector payoff is known to the decision-maker --- an assumption that we would not be ready
to make.

\begin{figure}[t!h!]
\center
\begin{minipage}{0.8\textwidth}
\rule{\linewidth}{.5pt}{}
{\small
\emph{Parameters}: the set $\cK$, a response function $\Psi : \cK \to \Delta(\cA)$ \smallskip \\
\emph{Initialization}: play an arbitrary $x_1 \in \Delta(\cA)$,
pick an arbitrary $\tm_1 \in \cK$ \smallskip \\
\emph{For} rounds $t = 2,3,\ldots$, \\
\begin{enumerate}
\item Update the discrepancy
\qquad $\displaystyle{
\Delta_{t-1} = \sum_{\tau=1}^{t-1} x_\tau \odot m_\tau -
\sum_{\tau=1}^{t-1} \Psi\bigl(\tm_s\bigr)\odot \tm_s\,;
}$
\item Play a mixed action
\qquad $\displaystyle{
x_t \in \argmin_{x \in \Delta(\cA)} \max_{m \in \cK} \,\, \langle \Delta_{t-1}, \, x \odot m \rangle\,;
}$
\item Compute
\qquad $\displaystyle{
\tm_t \in \argmax_{m \in \cK} \min_{x \in \Delta(\cA)} \langle \Delta_{t-1}, \, x \odot m \rangle\,.
}$
\end{enumerate}
}
\rule{\linewidth}{.5pt}
\end{minipage}

\caption{\label{fig:BSstrat} A generalization of the strategy of~\citet{BernShim}.}
\end{figure}

\begin{theorem}
\label{th:BS}
For all response functions $\Psi$,
the strategy of Figure~\ref{fig:BSstrat} is such that
for all $T \geq 1$, for all sequences $m_1,\,\ldots,\,m_T \in (\R^{d})^{\cA}$
of vectors of vector payoffs, possibly chosen by an opponent player,
\begin{equation}
\label{eq:whattobound4}
\left\Arrowvert \frac{1}{T} \sum_{t=1}^{T} x_t \odot m_t -
\frac{1}{T} \sum_{t=1}^{T} \Psi\bigl(\tm_t\bigr)\odot \tm_t \right\Arrowvert_2
\leq \frac{2 B_{2,\max}}{\sqrt{T}} \,.
\end{equation}
\end{theorem}

The obtained bound is deterministic and uniform over
all strategies of the opponent, just as the bound of Theorem~\ref{th:main} was.
Of course, the control~\eqref{eq:whattobound4} is a much weaker statement than
trying to force the convergence of the quantity~\eqref{eq:impossible} towards~$0$:
to which set can we guarantee that
\[
\sum_{t=1}^{T} \Psi\bigl(\tm_t\bigr)\odot \tm_t
\]
belongs? It seems difficult to relate this quantity to the set $\cC_{\varphi^\Psi(\ol{m}_T)}$ and get
the convergence $\dist_p\bigl(\ol{r}_T,\cC_{\varphi^\Psi(\ol{m}_T)}\bigr) \to 0$
except in some special cases. The applications of Section~\ref{sec:appli}
will further underline this limitation.

One of these special cases is when the set $\cC$ is approachable,
i.e., that the null target function $\varphi \equiv 0$ is achievable.
This assumption of approachability translates in our more general case into the existence of
a response function $\Psi_{\cC}$ such that $\Psi_{\cC}(m) \odot m \in \cC$ for all $m \in \cK$.
As advocated by~\citet{BernShim}, in such settings it is often computationally feasible
to access to $\Psi_{\cC}(m)$ and less costly than performing projections onto $\cC$.

In a nutshell, the strategy of~\citet{BernShim} can be extended to the setting of ``almost unknown''
games (the set $\cK$ needs to be known), but the obtained convergence guarantees
are meaningful only under an assumption of approachability of the target set $\cC$.
One of the two sources of unknownness of our setting is then (almost) dealt with:
the fact that the underlying structure of the game is unknown, but not the fact that the
target is unknown as well. \medskip

\proofof{of Theorem~\ref{th:BS}}
The construction of the strategy at hand and the proof of its performance bound
also follow some self-confident approach, as for Theorem~\ref{th:main};
however, no blocks are needed. We proceed as in~\eqref{eq:innprodthmain}
by developing the square Euclidian norm of $\Delta_{t+1}$
to relate it to the one of $\Delta_{t}$, where $t \geq 1$:
\begin{align*}
\Arrowvert \Delta_{t+1} \Arrowvert_2^2
= \Arrowvert \Delta_{t} \Arrowvert_2^2
& + 2 \underbrace{\left\langle \Delta_{t+1}, \,\, x_{t+1} \odot m_{t+1} -
\Psi\bigl(\tm_{t+1}\bigr)\odot \tm_{t+1} \right\rangle}_{\leq 0} \\
& +
\underbrace{\left\Arrowvert x_{t+1} \odot m_{t+1} - \Psi\bigl(\tm_{t+1}\bigr)\odot \tm_{t+1}
\right\Arrowvert_2^2}_{\leq 4 B_{2,\max}^2}\,.
\end{align*}
We show below that the inner product is non-positive, which after an immediate recurrence shows that
$\Arrowvert \Delta_{t+1} \Arrowvert_2^2 \leq 4 B_{2,\max}^2 (t+1)^2$ and concludes the proof.

Indeed, by von Neumann's minmax theorem, using the definitions of $x_{t+1}$ and $\tm_{t+1}$,
\begin{align*}
\max_{m \in \cK} \langle \delta_{t+1}, \, x_{t+1} \odot m \rangle
\min_{x \in \Delta(\cA)} \max_{m \in \cK} \langle \Delta_{t+1}, \, x \odot m \rangle
& = \max_{m \in \cK} \min_{x \in \Delta(\cA)} \langle \Delta_{t+1}, \, x \odot m \rangle \\
& = \min_{x \in \Delta(\cA)} \bigl\langle \delta_{t+1}, \, x \odot \tm_{t+1} \bigr\rangle\,.
\end{align*}
In particular, for all $m' \in \cK$ and $x' \Delta(\cA)$,
\[
\langle \delta_{t+1}, \, x_{t+1} \odot m' \rangle \leq
\max_{m \in \cK} \langle \delta_{t+1}, \, x_{t+1} \odot m \rangle
= \min_{x \in \Delta(\cA)} \bigl\langle \delta_{t+1}, \, x \odot \tm_{t+1} \bigr\rangle
\leq \bigl\langle \delta_{t+1}, \, x' \odot \tm_{t+1} \bigr\rangle\,.
\]
Choosing $m' = m_{t+1}$ and $x' = \Psi\bigl(\tm_{t+1}\bigr)$ entails
\[
\left\langle \Delta_{t+1}, \,\, x_{t+1} \odot m_{t+1} -
\Psi\bigl(\tm_{t+1}\bigr)\odot \tm_{t+1} \right\rangle \leq 0
\]
as used above to complete the induction.
\end{proof}

\subsubsection{Link with classical approachability, opportunistic approachability}
\label{sec:opportunistic}

We recall that in the setting of known finite games described in Section~\ref{sec:setting2},
vectors of vector payoffs $m$ actually correspond to the $r(\,\cdot\,,b)$.
This defines the closed convex set $\cK$ as the set of the $r(\,\cdot\,,y)$
for all mixed actions $y \in \Delta(\cB)$ of the opponent.
Both strategies considered therein relied on a response function $\Psi = x^\star$
defined as
\[
\forall \, y \in \Delta(\cB), \qquad
\Psi\big( r(\,\cdot\,,y) \big) = x^\star\big( r(\,\cdot\,,y) \big)
\in \argmin_{x \in \Delta(\cA)}  \, \dist_p\bigl(r(x,y),\,\cC\bigr)\,.
\]
Accessing to a value of this response function amounts to solving the convex program
\[
\min \left\| {\sum_{a \in \cA} \, x_a r(a,y)} - c\right\|^2\ \qquad
\mbox{s.t.} \quad x \in \Delta(\cA), \ \ c \in \cC\,,
\]
which can be done efficiently. (It even reduces to a quadratic problem when $\cC$ is a polytope.)

Our algorithm based on this response function approaches the set
$\cC_{\alpha_{\unif}}$, where the quantity $\alpha_{\unif}$ is defined in~\eqref{eq:alphaunif};
it is not required to compute the said quantity $\alpha_{\unif}$.
The same guarantee with the same remark apply to the two strategies
presented in Section~\ref{sec:setting2}: Blackwell's strategy for the case
$p=2$ only, and the strategy by~\citet{BernShim} for all $0 < p \leq +\infty$.
These three algorithms ensure in particular that the average payoffs $\ol{r}_T$ are
asymptotically inside of or on the border of the set $\cC_{\alpha_{\unif}}$.

Now, that $\alpha_{\unif}$ is null or positive indicates whether a convex set $\cC$ is approachable or not.
But the problem of determining the approachability of a set is actually an extremely difficult problem
as even the determination of the approachability of the singleton set $\cC = \{0\}$ in known games
is NP--hard to perform; see~\citet{Blkcomput}.
To see that there is no contradiction between being able to approach $\cC_{\alpha_{\unif}}$
and not being able to say that $\alpha_{\unif} > 0$ or not,
note that none of the algorithms discussed above does, neither in advance nor in retrospect, issue any statement
on the value of $\alpha_{\unif}$.
They happen to perform approachability to
$\cC_{\alpha_{\unif}}$ for the specific sequence of actions chosen by the opponent
but do not determine a minimal approachable set which would be suited for all sequences
of actions. In particular, they do not provide a certificate of whether a given convex set $\cC$ is approachable
or not.

\paragraph{Opportunistic approachability.}
In general, in known games, one has that the target function
considered above, $\varphi^{x^\star}$, satisfies $\varphi^{x^\star} \prec \alpha_{\unif}$.
That is, easy-to-control sequences of vectors $r(\,\cdot\,,b_t)$ can lead to an average
payoff $\ol{r}_T$ being much closer
to $\cC$ than the uniform distance $\alpha_{\unif}$: we get some pathwise refinement of classical
approachability. This should be put in correspondence with the recent, but different, notion of
opportunistic approachability (see~\citealp{Opport}).
However, quantifying exactly what we gain here with the pathwise refinement
would require much additional work (maybe a complete paper as the one mentioned above)
and this is why we do not explore further this issue.

\section{Applications}
\label{sec:appli}

In this section we work out two applications: learning while being evaluated with global cost functions,
and approachability under sample path constraints.

\subsection{Global cost functions}

This problem was introduced by~\citet{GlobalCosts} and slightly generalized by~\citet{BernShim}.
We first extend it to our setting of unknown games and describe what Theorem~\ref{th:main} guarantees in our case,
and then compare our approach and results to the ones of the two mentioned references.
We keep the original terminology of global costs (thus to be minimized) and do not switch to global gains
(to be maximized), but such a substitution would be straightforward.

\paragraph{Description of the problem in the case of unknown games.}
We denote by $\cKp \subset \R^d$ the closed convex and bounded set formed by the $m_a$ when $m \in \cK$ and $a \in \cA$.
A global cost function is a mapping $C : \cKp \to \R$ measuring the quality
of any vector in $\cKp$. For instance, the choice of a mixed action $x \in \Delta(\cA)$ given a vector of vector
payoffs $m \in \cK$ is evaluated by $C(x \odot m)$; or the performance of the average payoff $\ol{r}_T$ is
equal to $C(\ol{r}_T)$. Some regret is to be controlled to ensure that the latter quantity is small as well.
\citet{GlobalCosts} and~\citet{BernShim} defined this regret as
\begin{align}
\label{eq:globalcosts}
& C(\ol{r}_T) - \inf_{x \in \Delta(\cA)} C(x \odot \ol{m}_T)
= C(\ol{r}_T) - C^\star(\ol{m}_T) \\
\nonumber
\mbox{where} \quad \forall m \in \cK, \qquad & C^\star(m) = \inf_{x \in \Delta(\cA)} C(x \odot m)\,.
\end{align}
Assuming that $C$ is continuous, the infimum in the defining equation of $C^\star$ is achieved and we can
thus construct a response function $\Psi^\star : \cK \to \Delta(\cA)$ such that
\begin{equation}
\label{eq:defPsistarGC}
\forall m \in \cK, \qquad C \bigl( \Psi^\star(m) \odot m \bigr) = \min_{x \in \Delta(\cA)} C(x \odot m) = C^\star(m)\,.
\end{equation}
Actually, the proof techniques developed in the latter references (see the discussion below) only ensure a vanishing regret for
the convexification $\vex[C]$ of $C$ and the concavification $\cav[C^\star]$ of $C^\star$, i.e., they can only
issue statements of the form
\begin{equation}
\label{eq:vexcavregret}
\limsup_{T \to \infty} \, \bigl\{ \vex[C](\ol{r}_T) - \cav[C^\star](\ol{m}_T) \bigl\} \, \leq 0\,;
\end{equation}
they additionally get convergence rates when $\vex[C]$ is a Lipschitz function.

We recall that $\vex[C] \leq C$ and that $\cav[C^\star] \geq C^\star$,
so that the statements of the form above are much weaker than the original aim~\eqref{eq:globalcosts},
at least when $C$ is not convex or $C^\star$ is not concave.
A natural case when the latter assumptions are however satisfied is when $C = C_p$ is the $\ell_p$--norm, for $p > 1$
(including the supremum norm $p=+\infty$):
\[
C_p\bigl( (u_1,\ldots,u_d) \bigr) = \left( \sum_{j=1}^d u_j^p \right)^{\!\! 1/p}
\qquad \mbox{and} \qquad
C_{+\infty}\bigl( (u_1,\ldots,u_d) \bigr) = \max_{j=1,\ldots,d} |u_j|\,.
\]

\paragraph{Our main contribution: a better notion of regret.}
We will directly bound $C(\ol{r}_T)$, whether $C$ is convex or not,
and will similarly relax the assumption of concavity of $C^\star$ needed in all mentioned
references to tackle the desired regret~\eqref{eq:globalcosts}.

To that end, we propose a notion of regret that is better in
all cases (whether $C$ and $C^\star$ are respectively convex and concave, or not).
More precisely, we
compare $C(\ol{r}_T)$ to a quantity $\phi^\Psi(\ol{m}_T)$ based on any response function $\Psi$ and which generalizes the
definition~\eqref{eq:phiPsi}: for all $m \in \cK$,
\[
\phi^\Psi(m) = \sup \left\{ C\!\left(\sum_{i \leq N} \lambda_i \, \Psi(m_i) \odot m_i \right) : \ \
N \geq 1 \ \ \mbox{and}
\ \ \sum_{i \leq N} \lambda_i m_i = m \right\}.
\]
The extended notion of regret is then defined as $C(\ol{r}_T) - \phi^\Psi(\ol{m}_T)$.

We now explain why this new definition is always more ambitious than what could
be guaranteed so far by the literature, namely~\eqref{eq:vexcavregret}. Indeed,
when $C$ is convex and by definition of $\Psi^\star$, we have in particular
\[
\phi^{\Psi^\star}(m) \leq \sup \left\{ \sum_{i \leq N} \lambda_i \,
C \bigl( \Psi^\star(m_i) \odot m_i \bigr) : \ \
N \geq 1 \ \ \mbox{and}
\ \ \sum_{i \leq N} \lambda_i m_i = m \right\} = \cav[C^\star](m)\,.
\]
The inequality stated above can be strict.
For instance, as indicated in Section~\ref{sec:moreambitious}, when $C = \dist_p(\,\cdot\,,\,\cC\big)$ where $\cC$ is convex,
the global cost function $C$ is indeed convex. We then have
\[
\phi^{\Psi^\star} = \varphi^{x^\star}
\qquad \mbox{and} \qquad
\cav[C^\star] = \cav[\varphi^\star]\,,
\]
and thus we possibly have $\phi^{\Psi^\star} \prec \cav[C^\star]$, as stated in Lemma~\ref{lm:xstar}.
The function $C = \dist_p(\,\cdot\,,\,\cC\big)$ is also a Lipschitz function, which illustrates the interest
of the second part of the following corollary. We recall that $B_{2,\max}$
denotes the maximal Euclidean norm of elements in $\cK$.

\begin{corollary}
\label{cor:GC}
For all response functions $\Psi$,
when $C$ is continuous and convex, the strategy of Figure~\ref{fig:strat-blocks} ensures that,
uniformly over all strategies of the opponent,
\begin{equation}
\label{eq:GCregretgen}
\limsup_{T \to \infty} \, \Bigl\{ C(\ol{r}_T) - \phi^\Psi(\ol{m}_T) \Bigl\} \, \leq 0\,.
\end{equation}
When $C$ is in addition a Lipschitz function, with constant $L$ for the $\ell_2$--norm on $\cKp$,
we more precisely have
\[
C(\ol{r}_T) - \phi^\Psi(\ol{m}_T) \leq 10 L \, T^{-1/4} \ln A + 3 B_{2,\max} L \, T^{-1/2} \,.
\]
\end{corollary}

\begin{proof}
We apply Theorem~\ref{th:main} and use its notation.
The function $C$ is continuous thus uniformly continuous on the compact set $\cKp$.
Thus,
\[
\bigl\Arrowvert \ol{r}_T - c_T \bigr\Arrowvert_2 \longrightarrow 0
\qquad \mbox{entails} \qquad
C(\ol{r}_T) - C(c_T) \longrightarrow 0\,,
\]
both convergences toward~$0$ being uniform over all strategies of the opponent.
Now, by definition of $c_T$ as a convex combination of elements of the form $\Psi(m_i) \odot m_i$,
we have $C(c_T) \leq \phi^\Psi(\ol{m}_T)$, which concludes the first part of the corollary.

The second part is proved in the same manner, simply by taking into account the bound~\eqref{eq:whattobound}
and the fact that $C$ is a Lipschitz function.
\end{proof}

\paragraph{Discussion.}
\label{par:discussion}
As indicated in general in Section~\ref{sec:setting} we offered two extensions
to the setting of global costs: first, we explained how to deal with unknown games
and second, indicated what to aim for, given that the natural target is not
necessarily approachable and that sharper targets as the ones traditionally considered
can be reached. The second contribution is perhaps the most important one.

Indeed, the natural target~\eqref{eq:globalcosts} corresponds to ensuring the following convergence to a set:
\begin{equation}
\label{eq:goalGCappr}
(\ol{r}_T,\ol{m}_T) \, \longrightarrow \, \cH \qquad \mbox{where} \qquad \cH = \bigl\{ (r,m) : \ \ C(r) \leq C^\star(m) \bigr\}\,.
\end{equation}
This target set $\cH$ is not necessarily a closed, convex, and approachable set
but its convex hull $\conv[\cH]$ is so, as proved by~\citet{GlobalCosts} and~\citet{BernShim}.
This convex hull is exactly equal to
\[
\conv[\cH] = \bigl\{ (r,m) : \ \ \vex[C](r) \leq \cav[C^\star](m) \bigr\}\,.
\]
We replace the convergence of $(\ol{r}_T,\ol{m}_T)$ to the above convex hull $\conv[\cH]$ by a convergence to, e.g.,
the smaller set
\[
\Bigl\{ (r,m) : \ \ C(r) \leq \phi^{\Psi^\star}(m) \bigr\}\,.
\]
Such a convergence is ensured by~\eqref{eq:GCregretgen} and the continuity of~$C$,
and this set is smaller than $\conv[\cH]$ as follows from the discussion before Corollary~\ref{cor:GC}.

\citet{GlobalCosts} use directly Blackwell's approachability strategy to approach $\conv[\cH]$,
which requires the computation of projections onto $\conv[\cH]$, a possibly computationally delicate task.
We thus only focus on how \citet{BernShim} proceed and will explain why the obtained guarantee of convergence
to $\conv[\cH]$ cannot be easily improved with their strategy. We apply Theorem~\ref{th:BS}
to a lifted space of payoffs $\underline{\cK} \subset \R^d \times (\R^d)^{\cA}$.
Namely, with each $m \in \cK$, we associate
$\underline{m} \in \underline{\cK}$ defined as
\begin{equation}
\label{eq:lifting}
\forall \, a \in \cA, \qquad \underline{m}_a = \veccol{m_a}{m} \in \R^d \times (\R^d)^{\cA}\,.
\end{equation}
That is, the component $a \in \cA$ of $\underline{m}$ contains the corresponding component $m_a$ of $m$
as well as the vector $m$ itself. In particular,
\[
\frac{1}{T} \sum_{t=1}^{T} x_t \odot \underline{m}_t = \veccol{\ol{r}_T}{\ol{m}_T}\,.
\]
We pick the response function $\underline{\Psi}^\star : \underline{\cK} \to \Delta(\cA)$
corresponding to the base response function $\Psi^\star$ defined in~\eqref{eq:defPsistarGC}: $\underline{\Psi}^\star(\underline{m}) = \Psi^\star(m)$.
Then, the convergence~\eqref{eq:whattobound4} reads
\begin{equation}
\label{eq:lifting2}
\left\Arrowvert \veccol{\ol{r}_T}{\ol{m}_T} -
\frac{1}{T} \sum_{t=1}^{T} \underline{\Psi}^\star\bigl(\underline{\tm}_t\bigr)\odot \underline{\tm}_t \right\Arrowvert_2
=
\left\Arrowvert \frac{1}{T} \sum_{t=1}^{T} x_t \odot \underline{m}_t -
\frac{1}{T} \sum_{t=1}^{T} \underline{\Psi}^\star\bigl(\underline{\tm}_t\bigr)\odot \underline{\tm}_t \right\Arrowvert_2
\, \longrightarrow \, 0\,,
\end{equation}
for some $\underline{\tm}_1, \, \underline{\tm}_2, \, \ldots$ in $\underline{\cK}$.
By definition of $\underline{\cK}$ and $\underline{\Psi}^\star$, for all $t \geq 1$,
\[
\underline{\Psi}^\star\bigl(\underline{\tm}_t\bigr)\odot \underline{\tm}_t
= \veccol{\Psi^\star\bigl(\tm_t\bigr)\odot \tm_t}{\tm_t} \,\,
\in \bigl\{ (r,m) : C(r) \leq C^\star(m) \bigr\} \subseteq \cH \,.
\]
Thus, the convex combination of the $\underline{\Psi}^\star\bigl(\underline{\tm}_t\bigr)\odot \underline{\tm}_t$
belongs to $\conv[\cH]$ and the convergence~\eqref{eq:goalGCappr} is achieved. Under
additional regularity assumptions (e.g., continuity of $\vex[C]$ and $\cav[C^\star]$), the
stronger convergence~\eqref{eq:vexcavregret} holds as can be seen by adapting the arguments used in
the second part of Section~\ref{sec:B1}.

However, the limitations of the approach of~\citet{BernShim} are twofold.
First, as already underline in Section~\ref{sec:BS-gener}, the sets $\underline{\cK}$
or equivalently $\cK$ need to be known to the strategy; thus the game is not fully unknown.
Second, there is no control on where the $\underline{\tm}_t$ or $\tm_t$ lie, and therefore, there
is no reasonable hope to refine the convergence~\eqref{eq:goalGCappr} to a convergence to a set smaller
than $\conv[\cH]$ and defined in terms of $\ol{m}_T$ as in our approach.

\subsection{Approachability under sample path constraints}

We generalize here the setting of regret minimization in known finite
games under sample path constraints, as introduced by~\citet{Pathwise}
and further studied by~\citet{BernShim}. The straightforward
enough generalization is twofold: we deal with approachability
rather than just with regret; we consider unknown games.

\paragraph{Description of the problem in the case of unknown games.}
A vector in $\cKp \subset \R^d$ now not only represents some payoff but also some cost.
The aim of the player here is to control the average payoff vector (to have it converge
to the smallest expansion of a given closed convex target set~$\cP$) while abiding by
some cost constraints (ensuring that the average cost vector converges
to a prescribed closed convex set $\Gamma$).

Formally, two matrices $G$ and $C$, of respective sizes $g \times d$ and $\gamma \times d$,
associate with a vector $m_a \in \cKp \subset \R^d$ a payoff vector $G m_a \in \R^{g}$ and
a cost vector $C m_a \in \R^{\gamma}$. For instance, when the decision-maker chooses a mixed
action $x \in \Delta(\cA)$ and the vector of vector payoffs is $m \in \cK$,
she gets an instantaneous payoff $G(x \odot m)$ and suffers an instantaneous cost $C(x \odot m)$.
The admissible costs are represented by a closed convex set $\Gamma \subseteq \R^\gamma$,
while some closed convex payoff set $\cP \subseteq \R^g$ is to be approached.

The question is in particular what the decision-maker should aim for:
the target is unknown. Following the general aim~\eqref{eq:aim}
and generalizing the aims of~\citet{Pathwise} and~\citet{BernShim},
we assume that she wants the following convergences to take place, uniformly
over all strategies of the opponent: as $T \to \infty$,
\begin{equation}
\label{eq:aimSPC}
\dist_p \bigl( G \ol{r}_T, \, \cP_{\varphi(\ol{m}_T)} \bigr) \longrightarrow 0 \qquad \mbox{and}
\qquad \dist_p \bigl( C \ol{r}_T, \, \Gamma \bigr) \longrightarrow 0\,,
\end{equation}
for some target function $\varphi$ to be defined (being as small as possible).
That is, she wants to control her average payoff $G \ol{r}_T$ as well as she
can while ensuring that asymptotically, her average cost $C \ol{r}_T$ lies
in the set $\Gamma$ of admissible costs.

To make the problem meaningful and
as in the original references, we assume that the cost constraint is feasible.

\begin{assumption}
For all $m \in K$, there exists $x \in \Delta(\cA)$ such that $G(x \odot m) \in \Gamma$.
\end{assumption}

\paragraph{What the general result of Theorem~\ref{th:main} states.}
We consider mostly the following response function $x_\Gamma^\star$: for all $m \in K$,
\[
x_\Gamma^\star(m) \in \argmin \Bigl\{ \dist_p\bigl( G(x \odot m),\,\cP\bigr) :
\  x \in \Delta(\cA) \ \ \mbox{s.t.} \ \ G(x \odot m) \in \Gamma \Bigr\}\,,
\]
which provides the instantaneous-best and cost-abiding response.
The defining minimum is indeed achieved by continuity as
both $\cP$ and $\Gamma$ are closed sets.
Since in addition $\cP$ and $\Gamma$ are convex, the defining equation of $x_\Gamma^\star$
is a convex optimization problem under a convex constraint and can be solved efficiently.

Of course, more general (preferably also cost-abiding) response functions can be considered.
By a cost-abiding response function $\Psi$, we mean any response function such that
\[
\forall \, m \in \cK, \qquad C \bigl( \Psi(m) \odot m \bigr) \in \Gamma\,.
\]
This property is indeed satisfied by $x_\Gamma^\star$.

We adapt the definition~\eqref{eq:phiPsi} of the target function
based on some response function~$\Psi$ to only consider payoffs:
for all $m \in \cK$,
\[
\phi^\Psi(m) = \sup \left\{ \dist_p\!\left(G \sum_{i \leq N} \lambda_i \, \Psi(m_i) \odot m_i,\,\,\cP\right) : \ \
N \geq 1 \ \ \mbox{and} \ \ \sum_{i \leq N} \lambda_i m_i = m \right\}.
\]
A discussion below will explain why such goals, e.g., \eqref{eq:aimSPC}
with $\varphi = \phi^{x_\Gamma^\star}$, are more ambitious than the aims targeted
in the original references, which essentially consisted of shooting for
\eqref{eq:aimSPC} with $\varphi = \cav[\phi^\star]$ only and in restricted cases (uni-dimensional ones,
$g=1$), where for all $m \in \cK$,
\begin{equation}
\label{eq:phistarSPC}
\phi^\star(m) = \dist_p \Bigl( G \bigl( x_\Gamma^\star(m) \odot m \bigr),\,\cP \Bigr)\,.
\end{equation}

\begin{corollary}
\label{cor:SPC}
For all cost-abiding response functions $\Psi$,
the strategy of Figure~\ref{fig:strat-blocks} ensures that
for all $T \geq 1$ and for all strategies of the opponent,
\begin{align*}
\dist_p \bigl( G \ol{r}_T, \, \cP_{\phi^\Psi(\ol{m}_T)} \bigr)
\leq & \, L_G \bigl( 10 \, T^{-1/4} \ln A + 3 B_{2,\max} T^{-1/2} \bigr) \\
\phantom{\sum^{t}} \mbox{and} \qquad \qquad \qquad
\dist_p \bigl( C \ol{r}_T, \, \Gamma \bigr) \leq & \,
L_C \bigl( 10 \, T^{-1/4} \ln A + 3 B_{2,\max} T^{-1/2} \bigr)\,,
\end{align*}
where $L_G$, respectively, $L_C$,
is a norm on $G$, respectively, $C$, seen as a linear function from $\R^d$ equipped
with the $\ell_2$--norm to $\R^g$, respectively, $\R^\gamma$, equipped with the $\ell_p$--norm.

In particular, the aim~\eqref{eq:aimSPC} is achieved.
\end{corollary}

\begin{proof}
We apply Theorem~\ref{th:main} and use its notation.
By~\eqref{eq:whattobound} and by definition of $L_C$,
\[
\bigl\Arrowvert C \ol{r}_T - C c_T \bigr\Arrowvert_p \leq
L_C \bigl( 10 \, T^{-1/4} \ln A + 3 B_{2,\max} T^{-1/2} \bigr)\,.
\]
Because $\Psi$ was assumed to be cost-abiding and in view of the form~\eqref{eq:whattobound3}
of $c_T$, we have $C c_T \in \Gamma$ and we thus have proved
\[
\dist_p \bigl( C \ol{r}_T, \, \Gamma \bigr) \leq
L_C \bigl( 10 \, T^{-1/4} \ln A + 3 B_{2,\max} T^{-1/2} \bigr)\,.
\]
A similar argument, based on the fact that $G c_T \in \cP_{\phi^\Psi(\ol{m}_T)}$
by definition of~$\phi^\Psi$, yields the stated bound for
$\dist_p \bigl( G \ol{r}_T, \, \cP_{\phi^\Psi(\ol{m}_T)} \bigr)$.
\end{proof}

\paragraph{What the extension of earlier results, e.g., Theorem~\ref{th:BS}, yields.}
As indicated several times already, \citet{Pathwise} and~\citet{BernShim}
only considered the case of regret minimization, i.e., a special case of
approachability when $G$ is a linear form ($g = 1$) and $\cP$ is an interval of the form
$[G_\infty,+\infty)$ where $G_\infty$ is a bound on the values taken by~$G$. We will
discuss this special case below.

The strategies considered by~\citet{Pathwise} were not efficient (they relied on being able to
project on complicated sets or resorted to calibrated auxiliary strategies), unlike
the one studied by~\citet{BernShim}. We will thus focus on the latter.
The (not necessarily convex) target set considered therein is
\[
\cH = \bigl\{ (r,m) : \ \ Cr \in \Gamma \ \ \mbox{and} \ \ \dist_p(Gr,\cP) \leq \phi^\star(m) \bigr\}\,,
\]
where $\phi^\star$ was defined in~\eqref{eq:phistarSPC}.
Because $\cP$ is convex and $G$ is linear, the function $r \in \R^d \mapsto \dist_p(Gr,\cP)$
is convex; see, e.g., \citet[Example~3.16]{BoydV}.
The convex hull of $\cH$ thus equals
\[
\conv[\cH] = \bigl\{ (r,m) : \ \ Cr \in \Gamma \ \ \mbox{and} \ \ \dist_p(Gr,\cP) \leq \cav[\phi^\star](m) \bigr\}\,.
\]
To be able to compare the merits of the strategy by~\citet{BernShim}
to Corollary~\ref{cor:SPC}, we first extend
it to the case of unknown games, based on Theorem~\ref{th:BS}.
To that end we consider the same lifting as in~\eqref{eq:lifting} and apply
similarly Theorem~\ref{th:BS} to get~\eqref{eq:lifting2} as well,
for the cost-abiding response function $x^\star_\Gamma$.
Using that in this case, by definition of~$x^\star_\Gamma$,
\[
\underline{\Psi}^\star\bigl(\underline{\tm}_t\bigr)\odot \underline{\tm}_t
= \veccol{x^\star_\Gamma\bigl(\tm_t\bigr)\odot \tm_t}{\tm_t} \in \cH\,,
\]
the convergence~\eqref{eq:lifting2} rewrites
\[
\Biggl\Arrowvert \veccol{\ol{r}_T}{\ol{m}_T} -
\underbrace{\frac{1}{T} \sum_{t=1}^{T} \veccol{x^\star_\Gamma\bigl(\tm_t\bigr)\odot \tm_t}{\tm_t}}_{\in \, \conv[\cH]} \Biggr\Arrowvert_2
\, \longrightarrow \, 0
\]
and entails the convergence of $(\ol{r}_T,\ol{m}_T)$ to $\conv[\cH]$. In particular,
$C \ol{r}_T \to \Gamma$.
Under an additional regularity assumption, e.g., the continuity of $\cav[\phi^\star]$, we also get
(by adapting the arguments used in the second part of Section~\ref{sec:B1}) the stronger convergence
\begin{multline*}
\limsup_{T \to \infty} \, \Bigl\{
\dist_p(G \ol{r}_T,\cP) - \cav[\phi^\star](\ol{m}_T)
\Bigr\} \, \leq 0\,, \\
\mbox{that is,} \qquad
\limsup_{T \to \infty} \, \dist_p\bigl(G \ol{r}_T, \, \cP_{\cav[\phi^\star](\ol{m}_T)} \bigr) \, \leq 0\,.
\end{multline*}
Summarizing, the convergence~\eqref{eq:aimSPC} is guaranteed with $\varphi = \cav[\phi^\star]$;
an inspection of the arguments above shows that $\cav[\phi^\star]$ being actually uniformly
continuous, the desired uniformity over the strategies of the opponent is achieved.

The same limitations to this approach as mentioned at the end of the previous section arise
as far as the concepts of unknown game and unknown target are concerned.
First, the set $\cK$ needs to be known to the strategy and the game is not fully unknown.
Second, there is no control on where the $\tm_t$ lie, and therefore, there
is no reasonable hope to refine the convergence~\eqref{eq:aimSPC} with $\varphi = \cav[\phi^\star]$
into a convergence with a smaller target function $\varphi$.
In contrast, Corollary~\ref{cor:SPC} provided such a refinement with $\varphi = \phi^{x^\star_\Gamma}$,
which by convexity of $r \in \R^d \mapsto \dist_p(Gr,\cP)$ is smaller
and possibly strictly smaller than $\cav[\phi^\star]$
(adapt Lemma~\ref{lm:xstar} to prove the strict inequality).

\paragraph{A note on known games.}
However,
\citet[Section~5]{Pathwise} exhibit a class of cases when $\cav[\phi^\star]$
is the optimal target function: in known games, with scalar payoffs and scalar constraints,
and with set of constraints of the form $\Gamma = (-\infty,c_0]$. This amounts to minimizing
some constrained regret.

We thus briefly indicate what known games are in this context, as defined by~\citet{Pathwise}
and~\citet{BernShim}.
Some linear scalar payoff function $u : \Delta(\cA) \times \Delta(\cB) \to [0,G_\infty]$
and some linear vector-valued cost function $v : \Delta(\cA) \times \Delta(\cB) \to \R^\gamma$ are
given. (With no loss of generality we can assume that the
payoff function takes values in a bounded nonnegative interval.) The set $\cK$ of our general formulation
corresponds to the vectors, as $y$ describes $\Delta(\cB)$,
\[
\veccol{u(\,\cdot\,,y)}{v(\,\cdot\,,y)} \in \R \times \R^\gamma\,.
\]
The matrices $G$ and $C$ extract respectively the first component and all but the first component.
Regret is considered, that is, the payoff set $\cP$ to be be approached given the constraints is
$[G_\infty,\,+\infty)$. The expansions are $\cP_\alpha = [G_\infty-\alpha,\,+\infty)$.
The distance of some $r \in \R$ to some $\cP_\alpha$ equals $(G_\infty-\alpha - r)_+$.

In this context, convergences of the form~\eqref{eq:aimSPC} thus read
\begin{equation}
\label{eq:aimSPC-unidim}
\frac{1}{T} \sum_{t=1}^T v(x_t,b_t) \longrightarrow \Gamma \qquad \mbox{and} \qquad
\liminf_{T \to \infty} \, \biggl\{
\frac{1}{T} \sum_{t=1}^T u(x_t,b_t) -
\Bigl( G_\infty - \varphi \bigl( (\ol{u}_T,\ol{v}_T) \bigr) \Bigr)
\biggr\} \geq 0\,,
\end{equation}
where
\[
\veccol{\ol{u}_T}{\ol{v}_T} = \frac{1}{T} \sum_{t=1}^T \veccol{u(\,\cdot\,,b_t)}{v(\,\cdot\,,b_t)}\,,
\]
and thus correspond to some constrained regret-minimization problems.
Indeed, denoting
\[
\ol{y}_T = \frac{1}{T} \sum_{t=1}^T \delta_{b_t}
\]
the empirical frequency of actions $b_t \in \cB$ taken by the opponent,
and recalling that $u$ is bounded by $G_\infty$,
we have, for instance, when $\varphi = \phi^\star$,
\begin{multline}
\label{eq:aimSPC-unidim3}
\phi^\star \bigl( (\ol{u}_T,\ol{v}_T) \bigr) = G_\infty - u^\star_\Gamma(\ol{y}_T) \\
\mbox{where} \qquad
u^\star_\Gamma(y) = \max \bigl\{ u(x,y) : \  x \in \Delta(\cA) \ \ \mbox{s.t.} \ \ v(x,y) \in \Gamma \bigr\}\,.
\end{multline}
The convergence~\eqref{eq:aimSPC-unidim} finally reads when $\varphi = \phi^\star$:
\begin{equation}
\label{eq:aimSPC-unidim2}
\frac{1}{T} \sum_{t=1}^T v(x_t,b_t) \longrightarrow \Gamma \qquad \mbox{and} \qquad
\liminf_{T \to \infty} \, \biggl\{
\frac{1}{T} \sum_{t=1}^T u(x_t,b_t) -
u^\star_\Gamma(\ol{y}_T)
\biggr\} \geq 0\,.
\end{equation}
Just as we showed (in Section~\ref{sec:varphistar}) that in general the target function $\phi^\star$ is
not achievable, \citet[Section~3]{Pathwise} showed that the constrained regret with respect
to $u^\star_\Gamma(\ol{y}_T)$ defined in~\eqref{eq:aimSPC-unidim2} cannot be minimized.

The proposed relaxation was to consider its convexification $\vex[u^\star_\Gamma]$
instead in~\eqref{eq:aimSPC-unidim2}, which corresponds to $\cav[\phi^\star]$
in~\eqref{eq:aimSPC-unidim3}.
In this specific one-dimensional setting, the target function $\phi^{x^\star_\Gamma}$
equals $\cav[\phi^\star]$: our general theory provides no improvement. This is in line
with the optimality result for $\cav[\phi^\star]$ exhibited by \citet[Section~5]{Pathwise}
in this case.

\subsection{Approachability of an approachable set at a minimal cost}

This is the dual problem of the previous problem: have the vector-valued payoffs
approach an approachable convex set while suffering some costs and trying to control
the overall cost. In this case, the set $\cP$ is fixed and the $\alpha$--expansions
are in terms of the set of constraints $\Gamma$.
Actually, this is a problem symmetric to the previous one,
when the roles of $G$ and $\cP$ are exchanged with $C$ and $\Gamma$.

\acks{
Vianney Perchet acknowledges funding from the ANR, under grants ANR-10-BLAN-0112 and ANR-13-JS01-0004-01. Shie Mannor was partially supported by the ISF under contract 890015.
Gilles Stoltz would like to thank Investissements d'Avenir (ANR-11-IDEX-0003 / Labex Ecodec / ANR-11-LABX-0047)
for financial support. \medskip

An extended abstract of this article appeared in the Proceedings of the 27th Annual Conference on
Learning Theory (COLT'2014), \emph{JMLR Workshop and Conference Proceedings}, Volume 35, pages 339--355,
2014.}


\newpage
\bibliographystyle{plainnat}
\bibliography{Mannor-Perchet-Stoltz-MOL-Bib}

\newpage
\appendix

\section{Link with regret minimization in unknown games}
\label{sec:app:regretunknown}

The problem of regret minimization can be encompassed as an instance of approachability.
We recall here why the knowledge of the payoff
structure is not crucial for this very specific problem. (This, of course, is not the case at all
for general approachability problems.) \medskip

Indeed, with the notation of Section~\ref{sec:setting2}, the aim of regret minimization, in a known finite game
with payoff function $s : \cA \times \cB$, is for the decision-maker to ensure that
\[
\limsup_{T \to \infty} \,\, \frac{1}{T} \sum_{t=1}^T s(a_t,b_t)
- \max_{a \in \cA} \frac{1}{T} \sum_{t=1}^T s(a,b_t) \geq 0 \qquad \mbox{a.s.}
\]
This can be guaranteed by approaching $(\R_-)^{\cA}$ with the vector payoff function
$r : \cA \times \cB \to \R^{\cA}$ defined by
\begin{equation}
\label{eq:mtunknowngame}
r(a,b) = \bigl( s(a',b) - s(a,b) \bigr)_{a' \in \cA}\,.
\end{equation}
The necessary and sufficient condition for approachability of the closed convex
set $(\R_-)^{\cA}$ is satisfied for $r$.
The condition~\eqref{eq:howtoappr} rewrites in our case
\[
\forall \, y \in \Delta(\cB), \qquad
\bigl\langle
\widetilde{r}_t - \widetilde{c}_t, \,\, r(x_{t+1},y) - \widetilde{c}_t \bigr\rangle =
\Bigl\langle
\bigl( \widetilde{r}_t \bigr)_+, \,\,
r(x_{t+1},y) + \bigl( \widetilde{r}_t \bigr)_- \Bigr\rangle
=
\Bigl\langle \bigl( \widetilde{r}_t \bigr)_+, \,\, r(x_{t+1},y) \Bigr\rangle
\leq 0\,,
\]
where $(\,\cdot\,)_+$ and $(\,\cdot\,)_+$ denote respectively the vectors formed by
taking the nonnegative and non-positive parts of the original components of the
vector of interest. Now, using the specific form of $r$, we see that
\[
\Bigl\langle \bigl( \widetilde{r}_t \bigr)_+, \,\, r(x_{t+1},y) \Bigr\rangle
= \sum_{a' \in \cA} \bigl( \widetilde{r}_{t,a'} \bigr)_+ \, s(a',y)
- \left( \sum_{a' \in \cA} \bigl( \widetilde{r}_{t,a'} \bigr)_+ \right) s(x_{t+1},y)
\,.
\]
Either all components of $\widetilde{r}_{t}$ are non-positive, i.e., $\widetilde{r}_{t}$
is already in $(\R_-)^{\cA}$, or we can choose the mixed distribution $x_{t+1}$
defined by
\begin{equation}
\label{eq:mixeddistrregretunknown}
\forall \, a \in \cA, \qquad x_{t+1,a} = \frac{\bigl( \widetilde{r}_{t,a} \bigr)_+}{
\sum_{a' \in \cA} \bigl( \widetilde{r}_{t,a'} \bigr)_+}\,.
\end{equation}
In the latter case, we then get
\[
\forall \, y \in \Delta(\cB), \qquad \bigl\langle
\widetilde{r}_t - \widetilde{c}_t, \,\, r(x_{t+1},y) - \widetilde{c}_t \bigr\rangle
= \Bigl\langle \bigl( \widetilde{r}_t \bigr)_+, \,\, r(x_{t+1},y) \Bigr\rangle = 0\,,
\]
and~\eqref{eq:howtoappr} is in particular satisfied.

\paragraph{The knowledge of $s$ (or $r$) is not crucial here.}
Comments have to be made on the specific choice of $x_{t+1}$:
it is independent of the payoff structure $s$ (or $r$), it only depends on the past payoff vectors
$s(\,\cdot\,,b_\tau)$, where $\tau \leq t$.

In particular, the strategy above to minimize the regret can be generalized in a straightforward way
to the case of games with full monitoring but whose payoff structure is unknown.
In these games, at each round, the opponent chooses a payoff vector
\[
g_t = (g_{t,a})_{a \in \cA}\,,
\]
the decision-maker chooses an action $a_t \in \cA$ and observes the entire vector $g_t$,
while wanting to ensure that the regret vanishes,
\begin{equation}
\label{eq:regretunknowndef}
\limsup_{T \to \infty} \,\, \frac{1}{T} \sum_{t=1}^T g_{t,a_t}
- \max_{a \in \cA} \frac{1}{T} \sum_{t=1}^T g_{t,a} \geq 0 \qquad \mbox{a.s.}
\end{equation}
It suffices to replace all occurrences of $s(\,\cdot\,,b_t)$ above by $g_t$. In particular,
the payoff function $r$ defined in~\eqref{eq:mtunknowngame} is to be replaced by the
vectors of vector payoffs $m_t \in \R^{\cA \times \cA}$ whose components equal
\[
\forall \, a \in \cA, \qquad
m_{t,a} = \bigl( g_{t,a'} - g_{t,a} \bigr)_{a' \in \cA}\,.
\]

\paragraph{A note on the bandit monitoring: the case of unknown games.}
In the case of an unknown game (i.e., when the payoff structure is unknown and when only bandit monitoring
is available), the generic trick presented around~\eqref{eq:estmt}
should be adapted, as indicated by Footnote~\ref{fn:1}. Indeed, the only feedback
available at the end of each round is $g_{t,a_t}$ and not $m_{t,a_t}$.
The estimation to be performed is rather on the vectors $g_t$ than on the $m_t$:
for all $a \in \cA$,
\[
\widehat{g}_{t,a} = \frac{g_{t,a}}{x_{t,a}} \ind_{\{ a = a_t \}}\,,
\]
with the same constraints $x_{t,a} \geq \gamma_t$ for all $t \geq 1$,
from which we define
\[
\widehat{m}_{t,a} = \bigl( \widehat{g}_{t,a'} - \widehat{g}_{t,a} \bigr)_{a' \in \cA}\,.
\]
Substituting the estimates $\widehat{g}_{t}$
in the strategy defined around~\eqref{eq:mixeddistrregretunknown}
in lieu of the vectors $s(\,\cdot\,,b_t)$
ensures that the regret vanishes.

\newpage
\section{Calculations associated with Examples~1 and~2}

\subsection{Proof of Lemma~\ref{lm:notach}}
\label{sec:Expost-notach}

\proofof{for Example~1}
Assume by contradiction that the convergence~\eqref{eq:aim} can be achieved
and consider any strategy of the decision maker to do so, which we denote by $\sigma$.
It suffices to consider the almost sure convergence~\eqref{eq:aim}, the stronger
uniformity requirements stated after it will not be invoked. All statements in the sequel
hold almost surely, and quantities like $T_\varepsilon$ and
$T'_\varepsilon$ should be thought of as random variables.

Imagine in a first time that opponent chooses, at every stage $t \geq 1$, the vectors $m_t = m^\dag$.
We have $\varphi^\star(m^\dag) = 4$, the smallest of the supremum norms of $m^\dag_1$ and $m^\dag_2$.
The aim is then that the average payoffs $\ol{r}_{T}$ converge to $\cC_4$. But this can be guaranteed
only if the averages of the chosen mixed actions $x_t$ converge to $(1,0)$. That is,
given $\varepsilon > 0$, there exists some, possibly large, integer $T_\varepsilon$ such that
\begin{equation}
\label{eq:olrT}
\left\| \ol{r}_{T_\varepsilon} - \veccol{3}{4} \right\|_\infty \leq \varepsilon\,.
\end{equation}

Now, consider a second scenario. During the first $T_\varepsilon$ stages,
the opponent chooses the vectors $m_t = m^\dag$. By construction, as the strategy $\sigma$ is fixed,
\eqref{eq:olrT} is ensured. Now, in the next $T_\varepsilon$ stages,
for $T_\varepsilon +1 \leq t \leq 2 T_\varepsilon$,
assume that the opponent chooses the vectors $m_t = m^\sharp$,
and denote by $\gamma_\epsilon \in [0,1]$ the average of the first components $x_{t,1} \in [0,1]$
of the mixed actions $x_t$ selected by $\sigma$ in this second set of stages.
We have $\ol{m}_{2 T_\varepsilon} = m^{(1/2)}$, where
\[
m^{(1/2)} = \frac{1}{2} m^\dag + \frac{1}{2} m^\sharp
\qquad \mbox{has components} \qquad
m^{(1/2)}_1 = \left[\begin{array}{c} 3.5 \\ 3.5 \end{array}\right]
\quad \mbox{and} \quad
m^{(1/2)}_2 = \left[\begin{array}{c} 2.5 \\ 2.5 \end{array}\right].
\]
Therefore the target set is $\cC_{\varphi^\star(\ol{m}_{2 T_\varepsilon})}
= \cC_{2.5}$. However, by definition of $\gamma_\varepsilon$,
we have
\[
\ol{r}_{2 T_\varepsilon} = \frac{1}{2} \ol{r}_{T_\varepsilon} + \frac{1}{2} \sum_{t=T_\varepsilon+1}^{2 T_\varepsilon} x_t \odot m^\sharp
 = \frac{1}{2} \ol{r}_{T_\varepsilon} +
\frac{1}{2} \left( \gamma_\varepsilon \veccol{4}{3} + (1-\gamma_\varepsilon) \veccol{5}{0} \right)
 = \frac{1}{2} \ol{r}_{T_\varepsilon} +
\frac{1}{2} \veccol{5-\gamma_\varepsilon}{3\gamma_\varepsilon} \,,
\]
and therefore, because of~\eqref{eq:olrT},
\[
\left\| \ol{r}_{2 T_\varepsilon} - \veccol{4-\gamma_\varepsilon/2}{2+3\gamma_\varepsilon/2} \right\|_\infty \leq \frac{\varepsilon}{2}\,.
\]
This entails that
\[
\dist_\infty\Bigl( \ol{r}_{2 T_\varepsilon}, \, \cC_{\varphi^\star(\ol{m}_{2 T_\varepsilon})} \Bigr)
\geq \dist_\infty\!\left( \veccol{4-\gamma_\varepsilon/2}{2+3\gamma_\varepsilon/2}, \,\, \cC_{2.5} \right) - \frac{\varepsilon}{2}
\geq 1 - \varepsilon/2\,.
\]

This construction can be repeated again after stage $2 T_\varepsilon$, by choosing $m_t = m^\dag$
till a stage $T'_\varepsilon$ is reached when
\[
\left\| \ol{r}_{2 T_\varepsilon+T'_\varepsilon} - \veccol{3}{4} \right\|_\infty \leq \frac{\varepsilon}{2}\,;
\]
such a stage exists by the assumption that the convergence~\eqref{eq:aim}
is achieved by the strategy $\sigma$. One can then similarly see that
\[
\dist_\infty\Bigl( \ol{r}_{2 T_\varepsilon + 2 T'_\varepsilon}, \, \cC_{\varphi^\star(\ol{m}_{2 T_\varepsilon + 2 T'_\varepsilon})} \Bigr)
\geq 1 - \varepsilon/4\,.
\]
By repeating this over again and again, one proves that
\[
\limsup_{T \to \infty} \
\dist_\infty\Bigl( \ol{r}_{T}, \, \cC_{\varphi^\star(\ol{m}_{T})} \Bigr) \geq 1\,,
\]
which  contradicts  the assumption that $\sigma$ ensures the convergence~\eqref{eq:aim}.
The claim follows.
\end{proof}

\proofof{for Example~2 (sketch)}
The same construction as for the previous example holds, by
switching between a first regime when $m_t = (-1,1) \in \R^2$ is chosen and
at the end of which the average payoff should be close to null, $\ol{r}_{T_\varepsilon} \leq \varepsilon$.
Then, another regime of the same length starts with $m_t = (1,1) \in \R^2$ and no matter what
the decision-maker does, she will get an average payoff of $1$ in this regime.
In total, at the end of the second regime,
$\ol{r}_{2{T_\varepsilon}} \geq 1/2-\varepsilon$ while the target
set is given by
\[
\cC_{\varphi^\star(\ol{m}_{2{T_\varepsilon}})} =
\cC_{\varphi^\star((0,1))} = \cC_0 = \{0\}\,,
\qquad \mbox{as} \qquad
\ol{m}_{2{T_\varepsilon}} = (0,1) \in \R^2\,.
\]
This can be repeated over and over again.
\end{proof}

\subsection{Proof of Lemma~\ref{lm:cavnotamb}}
\label{sec:cavnotamb}

\proofof{for Example~1}
We have $\cav[\varphi^\star] \equiv 4$.
To prove this fact, we first
compute $\varphi^\star$. For $\nu \in [0,1]$, the components of
$m^{(\nu)} = \nu m^\dag + (1-\nu) m^\sharp$ equal
\begin{align*}
& m^{(\nu)}_1 = \nu \left[\begin{array}{c} 3 \\ 4 \end{array}\right] + (1-\nu) \left[\begin{array}{c} 4 \\ 3 \end{array}\right] = \left[\begin{array}{c} 4-\nu \\ 3+\nu \end{array}\right] \\
\mbox{and} \qquad &
m^{(\nu)}_2 = \nu \left[\begin{array}{c} 0 \\ 5 \end{array}\right] + (1-\nu) \left[\begin{array}{c} 5 \\ 0 \end{array}\right] = \left[\begin{array}{c} 5-5\nu \\ 5\nu \end{array}\right].
\end{align*}
Therefore,
\begin{equation}
\label{eq:varphistarEx1}
\varphi^\star\bigl(m^{(\nu)}\bigr) = \min \bigl\{ \max\{4-\nu,3+\nu\}, \,\,
\max\{5-5\nu,5\nu\} \bigr\}
= \begin{cases}
4 - \nu & \mbox{if} \ \nu \in [0, \, 1/4], \\
5 - 5\nu & \mbox{if} \ \nu \in [1/4, \, 1/2], \\
5 \nu & \mbox{if} \ \nu \in [1/2, \, 3/4], \\
3+\nu & \mbox{if} \ \nu \in [3/4, \, 1]. \\
\end{cases}
\end{equation}
We note that $\varphi^\star\bigl(m^\dag\bigr) = \varphi^\star\bigl(m^\sharp\bigr) = 4$
and that $\varphi^\star \leq 4$, so that $\cav[\varphi^\star]$ is identically
equal to $4$ on the set $\cK$ defined as the convex hull of $m^\dag$ and $m^\sharp$.

\begin{figure}[h!]
\begin{center}
\includegraphics{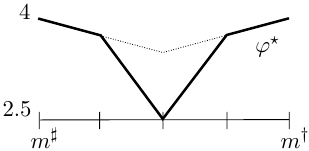}
\caption{\label{FG:Fig1-2} Graphs of the functions $\varphi^\star$ (bold solid line) and
$\varphi_1$ (dotted line).}
\end{center}
\end{figure}

Smaller target functions $\varphi$ such that the convergence~\eqref{eq:aim} holds
can be considered. This proves in particular that~\eqref{eq:aim} can also be guaranteed for the
larger $\cav[\varphi^\star]$.
Indeed,
\[
\varphi_1\bigl(m^{(\nu)}\bigr) = \max\{4-\nu,3+\nu\}
\]
is smaller than $\cav[\varphi^\star] \equiv 4$ (and even strictly smaller when
$\nu \ne 0$ and $\nu \ne 1$); see Figure~\ref{FG:Fig1-2}.
In addition, the convergence~\eqref{eq:aim} can hold for it.
Indeed, if the decision-maker plays $x_t = (1,0)$ at each round,
i.e., always picks the first component of $m_t$, then her average payoff equals
$\ol{r}_T = \ol{m}_{T,1}$, where $\ol{m}_{T} = m^{(\ol{\nu}_T)}$ for some $\ol{\nu}_T \in [0,1]$.
By definition of $\varphi_1$,
the distance of $m^{(\nu)}_1$ to $(-\infty,0]^2$ in the supremum norm is precisely $\varphi_1\bigl(m^{(\nu)}\bigr)$.
Therefore, we even have in this case
\[
\dist_\infty \bigl( \ol{r}_T, \, \cC_{\varphi_1(\ol{m}_T)} \bigr) = 0\,,
\]
which proves in particular that the convergence~\eqref{eq:aim} holds for $\varphi_1$.
\end{proof}

\proofof{for Example~2}
The computations are more involved in this seemingly simpler example.
As before, we start by computing $\varphi^\star$.
We refer to vectors $m \in [-1,1]^2$ chosen by the opponent as $m = (v,w)$ and to
the mixed actions picked by the decision-maker by $x = (x_1,1-x_1)$, where $x_1 \in [0,1]$.
The absolute value of a convex combination of $v$ and $w$ is to be minimized. This is achieved
with
\[
x^\star \bigl( (v,w) \bigr) = \left\{ \begin{array}{cl}
(1,0) & \mbox{if} \ 0 < v \leq w \ \mbox{or} \ 0 > v \geq w, \\
(0,1) & \mbox{if} \ 0 < w < v \ \mbox{or} \ 0 > w > v, \\
\displaystyle{\left( \frac{|w|}{|v|+|w|}, \,\, \frac{|v|}{|v|+|w|} \right)} & \mbox{if} \ v\,w \leq 0,
\end{array} \right.
\]
which leads to
\[
\varphi^\star \bigl( (v,w) \bigr) = \left\{ \begin{array}{cl}
\min\bigl\{|v|,\,|w|\bigr\} & \mbox{if} \ v\,w > 0, \\
0 & \mbox{if} \ v\,w \leq 0. \\
\end{array} \right.
\]
The concavification of $\varphi^\star$ thus admits the following expression: for all $v,w \in [-1,1]^2$,
\[
\cav[\varphi^\star]\bigl( (v,w) \bigr) = 1 - \frac{|v-w|}{2}\,.
\]
We replace a lengthy and tedious proof of this expression by the graphical illustrations
provided by Figure~\ref{fig:fig2}.

\begin{figure}[h!]
\begin{center}
\begin{tabular}{ccc}
  \includegraphics[scale=.85]{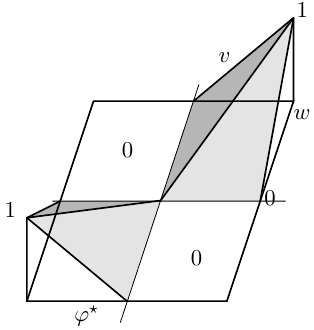}
& \includegraphics[scale=.85]{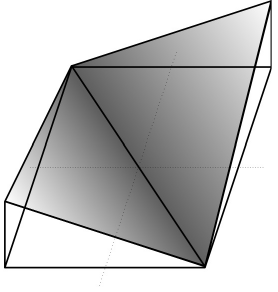}
& \includegraphics[scale=.85]{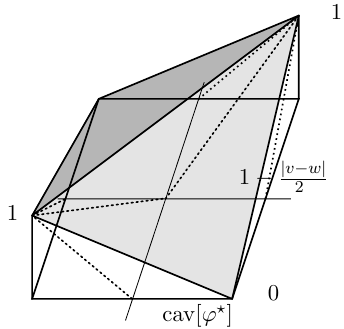}
\end{tabular}
\caption{\label{fig:fig2} Representations of $\varphi^\star$ (left),
of the alternative target function $\varphi_{1/2}$ (center),
and of the concavification $\cav[\varphi^\star]$ (right).}
\end{center}
\end{figure}

Now, we consider the target function $\varphi_{1/2}$ defined as
\begin{equation}
\label{eq:defvarphi12}
\varphi_{1/2} \bigl( (v,w) \bigr) = \frac{|v+w|}{2}\,.
\end{equation}
We denote $\ol{m}_T = \bigl( \ol{v}_T, \, \ol{w}_T \bigr)$.
By playing $x_t = (1/2,\,1/2)$ at each round, the decision-maker ensures that
\[
\ol{r}_T = \frac{\ol{v}_T + \ol{w}_T}{2}
\qquad \mbox{while} \qquad
\varphi_{1/2}\bigl(\ol{m}_T\bigr) = \left| \frac{\ol{v}_T + \ol{w}_T}{2} \right|.
\]
Thus, for this example again (in which the $\ell_p$--norm can be chosen freely), we have
\[
\dist_p \bigl( \ol{r}_T, \, \cC_{\varphi_{1/2}(\ol{m}_T)} \bigr) = 0\,.
\]
The convergence~\eqref{eq:aim} holds for $\varphi_{1/2}$ and also for
$\cav[\varphi^\star]$, since the latter is larger than $\varphi_{1/2}$ (strictly
larger at some points). Indeed,
the inequality
$\alpha_{1/2} \leq \cav[\varphi^\star]$ follows from the fact that
\[
|v+w|+|v-w| = \max\bigl\{ |v|,\,|w|\bigr\} \leq 2\,.
\]
That the inequality $\alpha_{1/2} \leq \cav[\varphi^\star]$
can be strict is seen, e.g., at $(0,0)$.
Again, an illustration is provided by Figure~\ref{fig:fig2}.
\end{proof}

\subsection{Proof of Lemma~\ref{lm:xstar}}
\label{sec:lm:xstar}

\proofof{for Example~2}
The proof below illustrates what we could call a ``sign compensation.''
We have here, for all $(v,w) \in \R^2$,
\begin{align*}
\cav[\varphi^\star]\bigl((v,w)\bigr) = & \sup \Bigg\{ \sum_{i=1}^N \lambda_i \, \bigl| x^\star\bigl( (v_i,w_i) \bigr) \odot (v_i,w_i) \bigr|  : \ \
N \geq 1 \ \ \mbox{and} \ \
\sum_{i=1}^N \lambda_i (v_i,w_i) = (v,w) \Bigg\}, \\
& \phantom{r} \\
\varphi^{x^\star}\bigl((v,w)\bigr) = & \sup \Bigg\{ \left| \sum_{i=1}^N \lambda_i \, x^\star \bigl( (v_i,w_i) \bigr) \odot (v_i,w_i) \right| : \ \
N \geq 1 \ \ \mbox{and} \ \
\sum_{i=1}^N \lambda_i (v_i,w_i) = (v,w) \Bigg\}.
\end{align*}
By a tedious case study consisting of identifying the worst convex decompositions
$\sum \lambda_i (v_i,w_i)$, one then gets the explicit expression
\begin{equation}
\label{eq:defvarphixstar}
{\varphi}^{x^\star}\bigl((v,w)\bigr) =
\left\{
\begin{array}{ll}
\bigl( 1+|v+w| \bigr)/3 & \mathrm{if}\ |2w-v|\leq 1 \ \mathrm{and}\  |2v-w|\leq 1,\\
(1+v)/2 & \mathrm{if}\ 2w-v \geq 1 \ \mathrm{and}\  v+w \geq 0,\\
(1+w)/2 & \mathrm{if}\ 2v-w \geq 1 \ \mathrm{and}\  v+w\geq 0,\\
(1-v)/2 & \mathrm{if}\ 2w-v \leq -1 \ \mathrm{and}\  v+w\leq 0,\\
(1-w)/2 & \mathrm{if}\ 2v-w \leq -1 \ \mathrm{and}\  v+w \leq 0,
\end{array}
\right.
\end{equation}
to be compared to the closed-form expression obtained earlier for $\cav[\varphi^\star]$,
namely
\[
\cav[\varphi^\star]\bigl( (v,w) \bigr) = 1 - \frac{|v-w|}{2}\,.
\]
Admittedly, a picture would help: we provide one as Figure~\ref{fig:3}.
\begin{figure}[h!]
\begin{center}
\begin{tabular}{cc}
\includegraphics{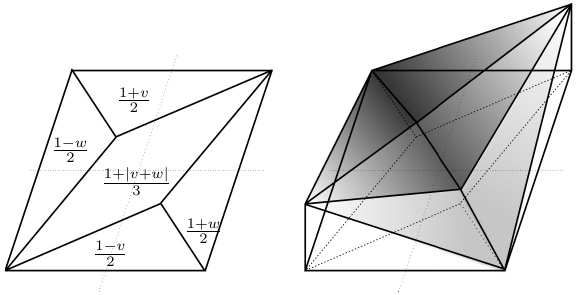} &
\includegraphics{Fig2-cavvarphistar}
\end{tabular}
\caption{\label{fig:3} Representations in 2D -- 3D of ${\varphi}^{x^\star}$ (left and center)
and 3D representation of $\cav[\varphi^\star]$.}
\end{center}
\end{figure}
We see (on the picture or by direct calculations) that ${\varphi}^{x^\star} \leq \cav[\varphi^\star]$,
and even that ${\varphi}^{x^\star} \prec \cav[\varphi^\star]$, by considering the respective
values $1/3$ and $1$ at $(0,0)$.
\end{proof}

\proofof{for Example~1}
We will prove that $\varphi^{x^\star} = \varphi_1$ so that the result will follow
from the inequality $\varphi_1 \prec \cav[\varphi^\star]$ already proved in Section~\ref{sec:cavnotamb}.

Indeed, as can be seen in the computations leading to the closed-form expression~\eqref{eq:varphistarEx1}
of $\varphi^\star$, we have
\[
x^\star\bigl(m^{(\nu)}\bigr) =
\begin{cases}
(1,0) & \mbox{if} \ \nu \in [0,\,1/4] \cup [3/4,\,1], \\
(0,1) & \mbox{if} \ \nu \in [1/4,\,3/4].
\end{cases}
\]
Therefore,
\[
x^\star\bigl(m^{(\nu)}\bigr) \odot m^{(\nu)} =
\begin{cases}
m^{(\nu)}_1 & \mbox{if} \ \nu \in [0,\,1/4] \cup [3/4,\,1], \\
m^{(\nu)}_2 & \mbox{if} \ \nu \in [1/4,\,3/4]. \\
\end{cases}
\]
But for $\nu \in [1/4,\,3/4]$, we have the component-wise inequality
\[
m^{(\nu)}_2 = \veccol{5 - 5\nu}{5\nu}
\leq  \veccol{4 - \nu}{3+\nu} = m^{(\nu)}_1 \,,
\]
which entails that for all $\nu \in [0,1]$, again component-wise, $x^\star\bigl(m^{(\nu)}\bigr) \odot m^{(\nu)} \leq m^{(\nu)}_1$.
Substituting in~\eqref{eq:phixstar} and using that the supremum distance to the negative orthant
is increasing with respect to component-wise inequalities, i.e.,
\[
(v,w) \in \R^2 \,\, \longmapsto \,\, \dist_\infty \bigl( (v,w), \, (-\infty,0]^2 \bigr) = \max \bigl\{ (v)_+, \, (w)_+ \bigr\}\,,
\]
we get that
\begin{align*}
\varphi^{x^\star}\bigl(m^{(\nu)}\bigr) & \ = \sup \Bigg\{ \dist_\infty\!\left( \sum_{i=1}^N \lambda_i \, x^\star\bigl(m^{(\nu_i)}\bigr) \odot m^{(\nu_i)}
,\,\,(-\infty,0]^2 \right) : \ \
N \geq 1 \ \ \mbox{and} \ \ \sum_{i=1}^N \lambda_i \nu_i = \nu \Bigg\} \\
& \ \leq \sup \Bigg\{ \dist_\infty\!\left( \sum_{i=1}^N \lambda_i \, m_1^{(\nu_i)} ,\,\,(-\infty,0]^2\right) : \ \
N \geq 1 \ \ \mbox{and} \ \
\sum_{i=1}^N \lambda_i \nu_i = \nu \Bigg\} \\
& \qquad \qquad = \dist_\infty\bigl( m_1^{(\nu)}, \,\, (-\infty,0]^2 \bigr)
= \varphi_1\bigl(m^{(\nu)}\bigr)\,.
\end{align*}

The converse inequality $\varphi^{x^\star} \geq \varphi_1$ follows from the decomposition of any
$\nu \in [0,1]$ as the convex combination of $1$, with weight $\lambda_1 = \nu$, and $0$, with weight $\lambda_2 = 1-\nu$.
In particular,
\begin{align*}
\varphi^{x^\star}\bigl(m^{(\nu)}\bigr) \geq & \,\,
\dist_\infty \Bigl( \nu \, x^\star\bigl(m^{(1)}\bigr) \odot m^{(1)}
+ (1-\nu) \, x^\star\bigl(m^{(0)}\bigr) \odot m^{(0)},\,\,(-\infty,0]^2 \Bigr) \\ = & \,\,
\dist_\infty\bigl( m_1^{(\nu)}, \,\, (-\infty,0]^2 \bigr) \,\, = \varphi_1\bigl(m^{(\nu)}\bigr) \,,
\end{align*}
as both $x^\star\bigl( m^{(0)} \bigr) = x^\star \bigl( m^{(1)} \bigr) = (1,0)$ as indicated above.
\end{proof}

\newpage
\section{Other technical proofs}

\subsection{Proof of two facts related to Definition~\ref{def:defach}}
\label{sec:B1}

The comments after Definition~\ref{def:defach} mentioned two facts, that we now prove.
First, that condition~\eqref{eq:cvgal} is less restrictive in general than~\eqref{eq:cv}.
Second, that for continuous target functions $\varphi$, the two definitions~\eqref{eq:cv}
and~\eqref{eq:cvgal} coincide, i.e., that that~\eqref{eq:cvgal} entails~\eqref{eq:cv}.

\paragraph{Condition~\eqref{eq:cvgal} is less restrictive in general than~\eqref{eq:cv}.}
Non-continuous target functions need to be considered to that end.
We consider a toy case when $\cA$ is reduced to one element (so that the decision-maker has no decision
to make and has to play this action), and the opponent player chooses elements in $\R$, i.e., $d=1$,
and more precisely, $\cK = [0,1]$. The target set equals $\cC = \{ 0 \}$ and the target function $\varphi$
is defined as
\[
\varphi(m)
= \begin{cases}
0 & \mbox{if} \ m \in [0,1), \\
1 & \mbox{if} \ m=1.
\end{cases}
\]
Since $d = A = 1$, we can identify $m_t$ and $r_t$.
We consider the sequence $m_t = 1-1/t$. We have
\[
(\ol{m}_T,\,\ol{r}_T) \longrightarrow
(1,1) \in \cG_\varphi\,, \qquad \qquad \mbox{since} \quad
\cG_\varphi = \bigl\{ (m,0) : m \in [0,1) \bigr\} \cup
\bigl( \{1\} \times [-1,1] \bigr) \,.
\]
On the contrary, for all $T$, we have $\cC_{\varphi(\ol{m}_T)} = \{0\}$
and therefore,
\[
\dist_p \bigl( \ol{r}_T, \, \cC_{\varphi(\ol{m}_T)} \bigr) = \ol{r}_T
\]
does not converge to $0$ as $T \to \infty$ (this sequence converges to $1$ actually).

\paragraph{Proof that~\eqref{eq:cvgal} entails~\eqref{eq:cv} under a continuity assumption.}
We consider a continuous function $\varphi : \cK \to [0,+\infty)$.
To show that~\eqref{eq:cvgal} entails \eqref{eq:cv}, it suffices to show that there exists a
function $f: (0,+\infty) \to (0,+\infty)$ with $f(\varepsilon) \to 0$ as $\varepsilon \to 0$,
such that for all $(m,r) \in \cK \times \R^d$,
\begin{equation}
\label{eq:eqvcontass}
\dist_p\bigl(r,\cC_{\varphi(m)}\bigr) \leq f\Bigl(\dist_p\bigl( (m,r),\cG_{\varphi}\bigr)\Bigr)\,;
\end{equation}
the required uniformities with respect to the strategies of the opponent will be carried over.

To that end, the continuity of $\varphi$ will be exploited through the following two properties.
First, $\cG_\varphi$ is closed.
Second, since $\cK$ is bounded, $\varphi$ is actually uniformly continuous:
we denote by $\omega : (0,+\infty) \to (0,+\infty)$
its modulus of continuity in the $\ell_p$--norm, which is a non-decreasing function that
satisfies $\omega(\varepsilon) \to 0$ as $\varepsilon \to 0$.

We denote by $(m_\cG,r_\cG)$ the projection in $\ell_p$--norm of $(m,r) \in \cK \times \R^d$
onto the closed set $\cG_\varphi$. By definition, $r_{\cG} \in \cC_{\varphi(m_\cG)}$ and
\begin{equation}
\label{eq:Hausd2}
\dist_p\bigl((m,r),\cG_\varphi\bigr) = \bigl\Arrowvert (m,r) - (m_\cG,r_\cG) \bigr\Arrowvert_p
\geq \Arrowvert m - m_{\cG} \Arrowvert_p\,.
\end{equation}
We also define an element $r'_{\cG} \in \cC_{\varphi(m)}$ as follows.
If $r_{\cG} \in \cC_{\varphi(m)}$ then we let $r'_{\cG} = r_{\cG}$.
Otherwise, $\varphi(m_\cG) - \varphi(m) > 0$ and as $r_{\cG} \in \cC_{\varphi(m_\cG)}$,
there exists an element $r_0 \in \cC$ such that $\Arrowvert r_{\cG} - r_0 \Arrowvert_p
\leq \varphi(m_\cG)$. (We recall that $\cC$ is a closed set and its expansions are closed expansions.)
We denote by $d$ the vector
\[
d = \frac{1}{\varphi(m_\cG)} (r_{\cG} - r_0)\,;
\]
by construction, $\Arrowvert d \Arrowvert_p \leq 1$. We introduce a new point $r'_{\cG}$
and provide a rewriting of $r_{\cG}$:
\[
r'_{\cG} = r_0 + \varphi(m) \, d \, \in \cC_{\varphi(m)}
\qquad \mbox{and} \qquad
r_{\cG} = r_0 + \varphi(m_\cG) \, d\,.
\]
These two equalities yield that
\[
\Arrowvert r_\cG - r'_\cG \Arrowvert_p  = \bigl( \varphi(m_\cG) - \varphi(m) \bigr) \, \Arrowvert d \Arrowvert_p
\leq \varphi(m_\cG) - \varphi(m)\,.
\]
Summarizing, we have in all cases (whether $r_{\cG}$ belongs to $\cC_{\varphi(m)}$ or not)
\begin{equation}
\label{eq:Hausd}
\Arrowvert r_\cG - r'_\cG \bigr\Arrowvert_p \leq \bigl| \varphi(m_\cG) - \varphi(m) \bigr| \, .
\end{equation}
Since $r'_{\cG} \in \cC_{\varphi(m)}$,
we get, by a triangle inequality,
\begin{eqnarray*}
\dist_p\bigl(r,\cC_{\varphi(m)}\bigr) & \leq &
\Arrowvert r - r'_{\cG} \bigr\Arrowvert_p
= \bigl\Arrowvert (m,r) - (m,r'_{\cG}) \bigr\Arrowvert_p \\
& \leq & \bigl\Arrowvert (m,r) - (m_{\cG},r_{\cG}) \bigr\Arrowvert_p
+ \bigl\Arrowvert (m_{\cG},r_{\cG}) - (m,r_{\cG}) \bigr\Arrowvert_p
+ \bigl\Arrowvert ( m,r_{\cG}) - (m,r'_{\cG}) \bigr\Arrowvert_p \\
& \leq & 2 \, \dist_p\bigl((m,r),\cG_\varphi\bigr)
+ \Arrowvert r_{\cG} - r'_{\cG} \Arrowvert_p\,,
\end{eqnarray*}
where the last inequality follows from~\eqref{eq:Hausd2}.
By~\eqref{eq:Hausd} and the uniform continuity of $\varphi$,
the last term in the right-hand side of the display above can be bounded as
\[
\Arrowvert r_{\cG} - r'_{\cG} \Arrowvert_p \leq
\bigl| \varphi(m_\cG) - \varphi(m) \bigr| \leq
\omega\bigl( \Arrowvert m_{\cG} - m \Arrowvert \bigr)
\leq \omega \Bigl( \dist_p\bigl((m,r),\cG_\varphi\bigr) \Bigr)\,,
\]
where for the last inequality, we used again~\eqref{eq:Hausd2}
and the fact that $\omega$ is non-decreasing.
Putting all pieces together, we proved~\eqref{eq:eqvcontass} with $f : x \mapsto 2x + \omega(x)$.

\subsection{A lemma used in the proof of Theorem~\ref{th:main}}
\label{sec:lemmaproofmainthm}

\begin{lemma}
Consider two positive numbers $\gamma_1,\gamma_2$ and form
the positive sequence $(u_n)$ defined by $u_1 = \gamma_2$ and
\[
u_{n+1} = u_n + 2 \gamma_1 \sqrt{(n+1)\, u_n} + \gamma_2 (n+1)^2
\]
for all $n \geq 1$. Then
\[
\forall \, n \geq 1, \qquad u_n \leq \max\bigl\{2\gamma_1^2,\,\gamma_2\bigr\} \, n^3\,.
\]
\end{lemma}

\begin{proof}
We proceed by induction and note that the relation is satisfied by construction for $n = 1$.
Assuming now that it holds for some $n \geq 1$, we show that it is also true for $n+1$.
Denoting $C = \max\bigl\{2\gamma_1^2,\,\gamma_2\bigr\}$, we get
\[
u_{n+1} = u_n + 2 \gamma_1 \sqrt{(n+1)\, u_n} + \gamma_2 (n+1)^2
\leq C\,n^3 + 2 \gamma_1 \sqrt{C} \sqrt{(n+1)\,n^3} + \gamma_2 (n+1)^2\,.
\]
It suffices to show that the latter upper bound is smaller than $C\,(n+1)^3$,
which follows from
\[
2 \gamma_1 \sqrt{C} \sqrt{(n+1)\,n^3} + \gamma_2 (n+1)^2
\leq \bigl( 2 \gamma_1 \sqrt{2C} +\gamma_2 \bigr) \,n^2 + 2\gamma_2\,n + \gamma_2
\leq 3C\,n^2 + 3C\,n + C\,;
\]
indeed, the first inequality comes from bounding $(n+1) n^3$ by $2 n^4$ and expanding the $(n+1)^2$ term,
while the second inequality holds because $C \geq \gamma_2$ and
$2C \geq 2 \gamma_1 \sqrt{2C}$ by definition of~$C$.
\end{proof}

\newpage
\section{Some thoughts on the optimality of target functions}
\label{sec:admiss}

We first define a notion of optimality,
based on the classical theory
of mathematical orderings, with $\prec$ (see Definition~\ref{def:prec}) being seen
as a strict partial order, with associated non-strict partial order
denoted by $\preccurlyeq$ (corresponding to the standard pointwise inequality $\leq$
for functions).

\subsection{On the existence of admissible target functions}

\begin{definition}
A target function $\varphi$ is admissible if it is achievable and
if there exists no other achievable
target function $\varphi'$ such that $\varphi' \prec \varphi$.
\end{definition}

There might exist several, even an infinite number of, admissible target functions,
as we will show below for Example~1. But there exists always at least one such
admissible function, as we show below in a non-constructive way. We unfortunately
were unable to exhibit general, concrete and closed-form admissible target functions.

\begin{lemma}
\label{lm:existadm}
In any unknown game there exists at least one admissible mapping.
\end{lemma}

\begin{proof}
The proof is
based on an application of Zorn's lemma.
We prove below that the
set $\cT$ of all achievable target functions $\varphi : \cK \to [0,+\infty)$,
which is partially ordered for $\preccurlyeq$,
has the property that every totally ordered subset $\cT_\Theta = \{ \varphi_\theta, \,\, \theta \in \Theta\}$
has a lower bound in $\cT$. In that case, Zorn's lemma ensures that the set $\cT$ contains at least one minimal element:
an element $\underline{\varphi}$ such that no other element $\varphi \in \cT$
satisfies $\varphi \prec \underline{\varphi}$.

Given a totally ordered subset $\cT_\Theta$, we can define the target function
\[
\varphi_\Theta : m \in \cK \longmapsto \inf_{\theta\in \Theta} \varphi_\theta(m)\,;
\]
$\varphi_\Theta$ is of course smaller than any element of $\cT_\Theta$. The point is to
show that $\varphi_\Theta \in \cT$, i.e., that $\varphi_\Theta$ is still achievable.

A property that we will use repeatedly below is that if two target functions are such that
$\varphi \preccurlyeq \phi$, then $\cG_{\varphi} \subseteq \cG_{\phi}$.

Now, by definition, the fact that the $\varphi_\theta$ are achievable means
that the compact sets $\cG_{\varphi_\theta}$ are each approachable
for the game with payoffs $(x,m) \in \Delta(\cA) \times \cK \mapsto (m, \, x \odot m)$;
in particular, they are non empty. The compact set
\[
\cG_{\varphi_\Theta} = \bigcap_{\theta \in \Theta} \, \cG_{\varphi_\theta}
\]
cannot be empty. Indeed, if it were, fixing any $\theta' \in \Theta$, we would have
that the subsets $\cG_{\varphi_{\theta'}} \setminus \cG_{\varphi_\theta}$
cover the compact topological space $\cG_{\varphi_{\theta'}}$.
As these subsets are open sets in the topological space $\cG_{\varphi_{\theta'}}$,
only finitely many of them would be needed for the covering, call them
$\cG_{\varphi_{\theta'}} \setminus \cG_{\varphi_j}$, with $j = 1,\ldots,N$.
Since $\cT_\Theta$ is totally ordered, one of the sets $\cG_{\varphi_j}$
is minimal for inclusion $\subseteq$, and therefore, one of the sets
$\cG_{\varphi_{\theta'}} \setminus \cG_{\varphi_j}$ is maximal for the inclusion $\subseteq$,
say, the one corresponding to $j=1$. Therefore, we would have
$\cG_{\varphi_{\theta'}} \setminus \cG_{\varphi_1} = \cG_{\varphi_{\theta'}}$.
As $\cT_\Theta$ is totally ordered, we would either have $\varphi_1 \preccurlyeq \varphi_{\theta'}$
and $\cG_{\varphi_1} \subseteq \cG_{\varphi_{\theta'}}$,
or $\varphi_{\theta'} \preccurlyeq \varphi_1$
and $\cG_{\varphi_{\theta'}} \subseteq \cG_{\varphi_1}$.
This would lead to $\cG_{\varphi_1} = \emptyset$ in the former case,
and $\cG_{\varphi_{\theta'}} = \emptyset$ in the latter case: in
both cases, to a contradiction.

In addition, we now prove that for all $\varepsilon > 0$,
there exists $\theta^\varepsilon \in \Theta$ such that
$\cG_{\varphi_{\theta^\varepsilon}}$ is included in the open $\varepsilon$--expansion of $\cG_{\varphi_\Theta}$,
which we denote by $\cG_{\varphi_\Theta,\varepsilon}$.
Indeed, denote by $\cH_{\varphi_\theta}$ the compact
sets $\cH_{\varphi_\theta} = \cG_{\varphi_\theta} \setminus \cG_{\varphi_\Theta,\varepsilon}$.
We have that
\[
\cH_{\varphi_\Theta} = \bigcap_{\theta \in \Theta} \, \cH_{\varphi_\theta}
= \cG_{\varphi_\Theta} \setminus \cG_{\varphi_\Theta,\varepsilon} = \emptyset\,.
\]
Therefore, by the same argument as above, we see that there must exist some
$\theta^\varepsilon$ such that $\cH_{\varphi_{\theta^\varepsilon}} = \emptyset$,
which is exactly what we wanted to prove.

So, summarizing, we
proved that $\cG_{\varphi_\Theta}$ is non empty and that each of its $\varepsilon$--open expansion is approachable (as it contains
an approachable set).
As in the proof of Lemma~\ref{lm:phixstarnotopti},
this means that $\cG_{\varphi_\Theta}$ is a $0$--approachable set, thus an approachable set,
or put differently, that $\varphi_\Theta$ is achievable.
\end{proof}

\subsection{Illustration on Examples~1 and~2}

Which response function $\Psi$ should we choose in practice?
And are target functions $\varphi^\Psi$ always admissible?
A convenient and natural choice in practice is $\Psi = x^\star$, but Example~2 shows
that unfortunately, $\varphi^{x^\star}$ is not always admissible.
Example~1 shows that many different target functions $\varphi^\Psi$ may be admissible.
It is thus difficult to issue any general theory on how to choose $\Psi$
and even on the optimality of the class of target functions $\varphi^\Psi$.

\paragraph{Example 2: unfortunately, $\varphi^{x^\star}$ is not admissible.}
Indeed, we have $\varphi_{1/2} \prec \varphi^{x^\star}$,
as can be seen by carefully comparing the closed-form expressions~\eqref{eq:defvarphi12}
and~\eqref{eq:defvarphixstar}. On the other hand, $\varphi_{1/2}$ is achievable:
it suffices to play $x_t = (1/2,\,1/2)$ at each round. Actually, $\varphi_{1/2}$ is of
the form $\varphi^\Psi$ for the constant response function $\Psi \equiv (1/2,\,1/2)$.

\paragraph{Example 1: all the target functions associated with the $\Psi_x \equiv (x,1-x)$ are admissible.}
We illustrate the general existence result of Lemma~\ref{lm:existadm} by
showing that in Example~1 the target functions $\varphi_x = \varphi^{\Psi_x}$ associated with the
constant response functions $\Psi_x \equiv (x,1-x)$ are admissible, for all $x \in [0,1]$. This
corresponds to the case when the decision-maker chooses the mixed action $(x,1-x)$ at all rounds.
In particular, the proof of Lemma~\ref{lm:xstar} indicates that $\varphi_1 = \varphi^{x^\star}$,
the latter is thus admissible unlike in Example~2. \smallskip

Closed-form expressions of these target functions $\varphi_x$ will be needed. For $\nu \in [0,1]$,
when the decision-maker plays $(x,1-x)$ while the vector of vector payoffs is $m^{(\nu)}$, she gets
an average payoff which we denote by $r\bigl((x,1-x),\, m^{(\nu)} \bigr)$ and which equals
\begin{align*}
r\bigl((x,1-x),\,m^{(\nu)}\bigr) =
& \,\, x \bigl( \nu \, m_1^\dag + (1-\nu) m_1^\sharp \bigr) + (1-x)
\bigl( \nu \,m_2^\dag + (1-\nu) m_2^\sharp \bigr) \\ = & \,\,
x \left( \nu \veccol{3}{4} + (1-\nu) \veccol{4}{3} \right) + (1-x)
\left( \nu \veccol{0}{5} + (1-\nu) \veccol{5}{0} \right) \\
= & \
\veccol{x(4-\nu) + 5(1-x)(1-\nu)}{x(3+\nu) + 5(1-x)\nu} =
\veccol{5-x - \nu(5-4x)}{3x+\nu(5-4x)}\,.
\end{align*}
The underlying response function $\Psi_x$ being constant, no convex decomposition
needs to be considered in the defining supremum for $\varphi_x\bigl(m^{(\nu)}\bigr)$ and the latter
equals
\[
\varphi_x\bigl(m^{(\nu)}\bigr) = \dist_\infty\Bigl( r\bigl((x,1-x),\,m^{(\nu)}\bigr), \,\, (-\infty,0]^2 \Bigr)
= \max \big\{ 5-x - \nu(5-4x),\,\,3x+\nu(5-4x)\bigr\}\,.
\]
Since $\nu \mapsto 5-x - \nu(5-4x)$ is decreasing, and $\nu \mapsto 3x+\nu(5-4x)$ is increasing,
and both functions take the same value $5/2 - x/2$ at $\nu = 1/2$, we get
\[
\varphi_x\bigl(m^{(\nu)}\bigr) =
\begin{cases}
5-x - \nu(5-4x) & \mbox{if} \ \nu \in [0, \, 1/2], \\
3x+\nu(5-4x) & \mbox{if} \ \nu \in [1/2, \, 1].
\end{cases}
\]

Our proof follows the methodology used to prove Lemma~\ref{lm:notach}. We fix any
strategy of the decision-maker achieving a target function $\varphi \leq \varphi_x$, for some fixed $x \in [0,1]$,
and we show that necessarily, $\varphi = \varphi_x$. We provide a detailed proof of the equality only for $m^{(\nu)}$
where $\nu$ lies in the interval $(3/4, \,1]$; but this proof
can be adapted in a straightforward manner to prove the equality as well the intervals
$[0, \,1/4]$, $[1/4, \,1/2]$ and $[1/2, \,3/4]$. As in the proof of Lemma~\ref{lm:notach},
it suffices to consider the almost sure statement of convergence as in~\eqref{eq:aim}; the uniformity with
respect to strategies of the opponent is not needed. All statements below hold almost surely
and the times $T$ and $T'$ should be thought of as random variables.

Our argument for $\nu \in [3/4, \,1]$ is based on three sequences of mixed actions.
For the first one, assume that the opponent chooses $m^{\dag}$ (corresponding to $\nu=1$)
during $T$ stages, where $T$ can be made arbitrarily large.
We denote by $(v_T,1-v_T)$ the average of the mixed actions $(x_t,1-x_t)$ played by the decision-maker during these rounds.
The average payoff vector received equals
\[
\veccol{3 v_T}{5-v_T}\,,
\]
whose distance to the negative orthant is $5-v_T$.
Since $\varphi_x(m^{\dag}) = 5-x$ and the strategy achieves $\varphi(m^{\dag})$,
where by assumption $\varphi(m^{\dag}) \leq \varphi_x(m^{\dag})$,
it holds that $\limsup 5-v_T \leq 5-x$ as $T \to \infty$. For the sake of compactness, we will denote this fact by
$5-v_T \lesssim 5-x$. This entails that $\liminf v_T \geq x$ as $T \to \infty$,
a fact that we denote by $v_T \gtrsim x$.

During the next $T$ stages, we
assume that the opponent chooses $m^{\sharp}$ (which corresponds to $\nu = 0$) and denote by $(w_T,1-w_T)$
the average of the mixed actions $(x_t,1-x_t)$ played by the decision-maker during these rounds.
The average payoff vectors received between rounds $T+1$ to $2T$, on the one hand,
and during rounds $1$ to $2T$, on the other hand, are therefore respectively equal to
\[
\veccol{5-w_T}{3 w_T} \qquad \mbox{and} \qquad
\frac{1}{2} \veccol{3 v_T + 5 - w_T}{5 - v_T + 3 w_T}\,,
\]
so that the distance of the latter to the negative orthant is given by
\[
\frac{1}{2} \max\bigl\{ 3 v_T + 5 - w_T, \,\, 5 - v_T + 3 w_T \bigr\}\,,
\]
which we know is asymptotically smaller than $\varphi\bigl(m^{(1/2)}\bigr)$ by achievability of $\varphi$, where
by assumption
\[
\varphi\bigl(m^{(1/2)} \bigr) \leq \varphi_x \bigl( m^{(1/2)} \bigr) = 5/2 + x\,.
\]
We thus obtained the following system of equations:
\[
\left\{
\begin{array}{rl}
v_T & \gtrsim x, \\
3 v_T - w_T & \lesssim 2x, \\
- v_T + 3 w_T & \lesssim 2x.
\end{array}
\right.
\]
The sum of the last two inequalities is $v_T + w_T \lesssim 2x$. Together with the first inequality
$v_T \gtrsim x$, it leads to $w_T \lesssim x$. Substituting in the second inequality,
we get $3 v_T - w_T \approx 2x$, where the $\approx$ symbol denotes a convergence: $3 v_T - w_T \to 2x$
as $T \to \infty$. Summing the proved limits $- 9 v_T + 3 w_T \approx -6 x$ and $8 v_T \gtrsim 8 x$
yields $- v_T + 3 w_T \gtrsim 2 x$, thus $- v_T + 3 w_T \approx 2 x$. From the latter
limit and $3 v_T - w_T \approx 2x$, we finally get $v_T \approx x$ and $w_T \approx x$.

Consider now some $\nu > 3/4$. We show that $\varphi(\nu) \geq \varphi_x(\nu)$.
To that end, assume that after the $T$ stages of $m^\dag$
the opponent switches instead to $m^{(3/4)} = (3/4) \, m^{\dag} + (1/4) \, m^\sharp$ during
\[
T' = \frac{1-\nu}{\nu-3/4} T
\]
rounds. Note that in this case, the average values of the coefficients for $m^{\dag}$ and $m^{\sharp}$ used
in the first $T+T'$ rounds are proportional to
\begin{equation}
\label{eq:freqnu}
1 + \frac{3}{4} \, \frac{1-\nu}{\nu-3/4} = \frac{\nu}{4\nu - 3}
\qquad \mbox{and} \qquad
0 + \frac{1}{4} \, \frac{1-\nu}{\nu-3/4} = \frac{1-\nu}{4\nu - 3}
\end{equation}
that is, $\ol{m}_{T+T'} = m^{(\nu)} = \nu \, m^\dag +(1-\nu) \,m^\sharp$ was played.
We perform first some auxiliary calculations: by multiplying the equalities in~\eqref{eq:freqnu}
by $T$, we see that the total number $T+T'$ of rounds equals $T+T' = T/(4\nu-3)$.
In particular, we have
\[
T/(T+T') = 4\nu-3 \qquad \mbox{and} \qquad T'/(T+T') = 1 - (4\nu-3) = 4 - 4\nu\,.
\]
Finally, denoting by $u_{T'}$ the average mixed action played by the decision-maker in rounds
$T+1$ to $T+T'$, we have that the average vector payoffs during rounds
$T+1$ to $T+T'$ and during rounds $1$ to $T+T'$ are respectively equal to
\[
r\bigl( (u_{T'},1-u_{T'}),\,3/4 \bigr) = \veccol{5-u_{T'}-(3/4)(5-4u_{T'})}{3u_{T'}+(3/4)(5-4u_{T'})} = \veccol{2u_{T'}+5/4}{15/4}
\]
and
\begin{equation}
\label{eq:veccolorth}
\frac{T}{T+T'}\,\veccol{3 v_T}{5-v_T} +
\frac{T'}{T+T'}\,\veccol{2u_{T'}+5/4}{15/4}
= (4\nu-3) \,\veccol{3 v_T}{5-v_T} + (4-4\nu)\,\veccol{2u_{T'}+5/4}{15/4}\,.
\end{equation}
The overall average payoff is given by the distance of this vector in the supremum norm to the
negative orthant and must be smaller than $\varphi\bigl(m^{(\nu)}\bigr)$ in the limit, by achievability of $\varphi$.
However, the said distance of~\eqref{eq:veccolorth} to the negative orthant is bound to be larger than the
second component of~\eqref{eq:veccolorth}, which equals
\[
(4\nu-3)(5-v_T) + (4-4\nu) \frac{15}{4} \,\, \longrightarrow \,\,
(4\nu-3)(5-x) + (4-4\nu) \frac{15}{4} =
5\nu+3x-4\nu x = \varphi_x\bigl(m^{(\nu)}\bigr)\,,
\]
as $T \to \infty$, where we substituted the above limit $v_T \to x$.

We thus proved $\varphi\bigl(m^{(\nu)}\bigr) \geq \varphi_x\bigl(m^{(\nu)}\bigr)$, as claimed.

\end{document}